\documentclass{article}

\usepackage[main, final]{neurips_2026}

\usepackage[utf8]{inputenc} %
\usepackage[T1]{fontenc}    %
\usepackage{hyperref}       %
\usepackage{url}            %
\usepackage{booktabs}       %
\usepackage{amsfonts}       %
\usepackage{nicefrac}       %
\usepackage{microtype}      %
\usepackage{xcolor}         %
\usepackage{graphicx}
\usepackage{multirow}
\usepackage{amsmath}
\usepackage{xspace}
\usepackage{wrapfig}
\usepackage{amssymb}
\usepackage{cleveref}
\usepackage{pifont}
\usepackage{fontawesome5}

\usepackage{amsthm}

\usepackage{enumitem}

\newcommand{\eg}{\hbox{\emph{e.g.,}}\xspace}
\newcommand{\ie}{\hbox{\emph{i.e.,}}\xspace}

\newcommand{\ours}{Preference Delta Aggregation\xspace}
\newcommand{\shortours}{PDA\xspace}
\newcommand{\ourmerging}{Geometric Alignment Merging\xspace}
\newcommand{\shortmerging}{GAM\xspace}
\newcommand{\tulu}{Tülu3\xspace}

\title{From ``Weak'' Signals to Strong Models: Preference Delta Aggregation with LoRA Merging}

\author{%
  Qi Sun$^{1}$\thanks{Equal contribution.} \quad
  Siyue Zhang$^{2}$\footnotemark[1]\textsuperscript{ }\thanks{Corresponding authors: Chen Zhao (\texttt{cz1285@nyu.edu}) and Siyue Zhang (\texttt{siyue001@e.ntu.edu.sg}).} \quad
  Yulin Chen$^3$ \quad
  Yuxiang Xue$^1$ \quad
  Ru Peng$^4$ \quad
  Chen Zhao$^1$\footnotemark[2] \\
  \\
  $^1$NYU Shanghai \quad $^2$NTU \quad $^3$NYU \quad $^4$ZJU \\
  \\
  \texttt{\{qs2196, yc7320, yx3044, cz1285\}@nyu.edu} \\
  \texttt{siyue001@e.ntu.edu.sg} \quad \texttt{rupeng@zju.edu.cn} \\
  \\
  \vspace{-2mm} 
  \href{https://github.com/AlbertQiSun/Preference-Delta-Aggregation}{\faGithub\ \texttt{AlbertQiSun/Preference-Delta-Aggregation}}
}

\begin{document}

\maketitle

\begin{abstract}
    
    Training strong large language models (LLMs) requires high-quality supervision, which is often scarce.
    Recent work shows that paired preference data from weak-weaker model pairs (\eg Qwen3 4B over 1.7B), despite the limited quality of individual responses, can provide an effective supervision signal through relative quality deltas, which we term a ``weak'' signal.
    This motivates a key research question: can multiple ``weak'' signals be constructively aggregated for improving strong models (\eg Qwen3 8B)?
    To this end, we propose \ours\ (\shortours), the first framework that derives a preference delta from each weak-weaker model pair, instantiates it as a LoRA adapter learned through preference optimization, and aggregates the resulting deltas via LoRA merging.
To improve the composition of independently trained adapters, we introduce \ourmerging (\shortmerging), which aligns their left and right singular bases through orthogonal Procrustes rotations and averages the aligned factors.
    Evaluations on knowledge reasoning and agentic search show that \shortours-\shortmerging exceeds the strongest individual delta, with further gains along the evaluated sequences of additional deltas.
    Using three reasoning deltas and four agentic deltas, \shortours-\shortmerging improves Qwen3-8B by 6.8 and 7.3 points on average, respectively, exceeding the strongest single-delta baseline by 2.1 and 4.3 points.
    These configurations achieve the highest average scores among the evaluated merging methods, with similar gains observed on a second student model.
    Further analysis suggests that these gains come from composing complementary capabilities encoded across distinct preference deltas.\footnote{Code is available at \url{https://github.com/AlbertQiSun/Preference-Delta-Aggregation}.}

\end{abstract}

\section{Introduction} \label{sec:introduction}

High-quality data is widely recognized as a key ingredient in building strong large language models (LLMs). Accordingly, substantial prior work has focused on curating training data throughout the pipeline, spanning pretraining \citep{datacomp,fineweb,olmo} and post-training \citep{nemotron,toolmind,mdr}. However, many desirable tasks remain difficult to supervise effectively, either due to prohibitive annotation costs or because they exceed human expertise. This challenge has motivated growing interest in leveraging low-quality data to expand the frontier of LLM capabilities \citep{weaktostrong,varying}. Building on this direction, recent work introduces the Delta Learning Hypothesis \citep{geng2025deltalearninghypothesispreference}, which posits that relative quality differences between weak responses (\eg Qwen3 4B over\ 1.7B) can serve as effective ``weak'' supervision signals for improving strong models (\eg Qwen3-8B) through preference tuning.

The success of the Delta Learning Hypothesis raises a question: can multiple such ``weak'' signals be aggregated to yield further gains? We first examine sequential preference optimization over multiple datasets and joint optimization on their mixture. In our experiments, sequential training falls below the untuned student, while joint training falls below the strongest single delta. The agentic joint baseline uses the same preference datasets and a similar total number of optimization steps as our aggregation framework. These results motivate an alternative to combining the datasets during training.

Therefore, we propose the \ours\ framework (\shortours), which independently preference-tunes the strong student on datasets from different weak-weaker model pairs using LoRA \citep{hu2021loralowrankadaptationlarge}, as illustrated in \Cref{fig:figure1}. Each adapter encodes the parameter update learned from one pair, which we term a \emph{preference delta}. We aggregate these updates through LoRA merging \citep{yadav2023tiesmergingresolvinginterferencemerging,stoica2024modelmergingsvdtie,lorahub}. Independent training retains a separate update for each source pair, allowing us to combine the learned signals after optimization. All adapters start from the same student and modify the same target modules, making them compatible for parameter-space merging. This separates the question of whether a source pair provides useful supervision from the question of how multiple learned updates should be composed. The merging operator determines how the adapters contribute to the resulting model. We introduce \ourmerging\ (\shortmerging), which decomposes each update by singular value decomposition (SVD), aligns its left and right singular bases to a reference, and averages the aligned factors.

Following prior work \citep{geng2025deltalearninghypothesispreference}, we evaluate the proposed \shortours and \shortmerging on knowledge reasoning tasks, and further extend the evaluation to more challenging agentic search settings requiring both reasoning and retrieval \citep{diffusion,mrmr}. We construct preference datasets from GSM8K \citep{cobbe2021trainingverifierssolvemath} and MuSiQue \citep{trivedi2022musiquemultihopquestionssinglehop} using multiple weak-weaker model pairs drawn from diverse model families, including Llama-3.2 \citep{grattafiori2024llama3herdmodels}, Qwen3 \citep{yang2025qwen3technicalreport}, DeepSeek-R1 \citep{deepseek}, AceSearcher \citep{xu2025acesearcherbootstrappingreasoningsearch}, and Search-R1 \citep{jin2025searchr1trainingllmsreason}. We evaluate strong student models, including Qwen3-8B and \tulu-8B \citep{lambert2025tulu3pushingfrontiers}, across diverse reasoning benchmarks as well as both single-hop and multi-hop search QA benchmarks, and compare against four categories of baselines: the original student models without fine-tuning, single-delta preference tuning, training-based multi-delta aggregation methods (\eg sequential and joint training), and existing LoRA merging approaches such as naive averaging \citep{ilharco2023editingmodelstaskarithmetic}, TIES-Merging \citep{yadav2023tiesmergingresolvinginterferencemerging}, and KnOTS \citep{stoica2024modelmergingsvdtie}.

\begin{figure}
    \centering
    \includegraphics[width=1\linewidth]{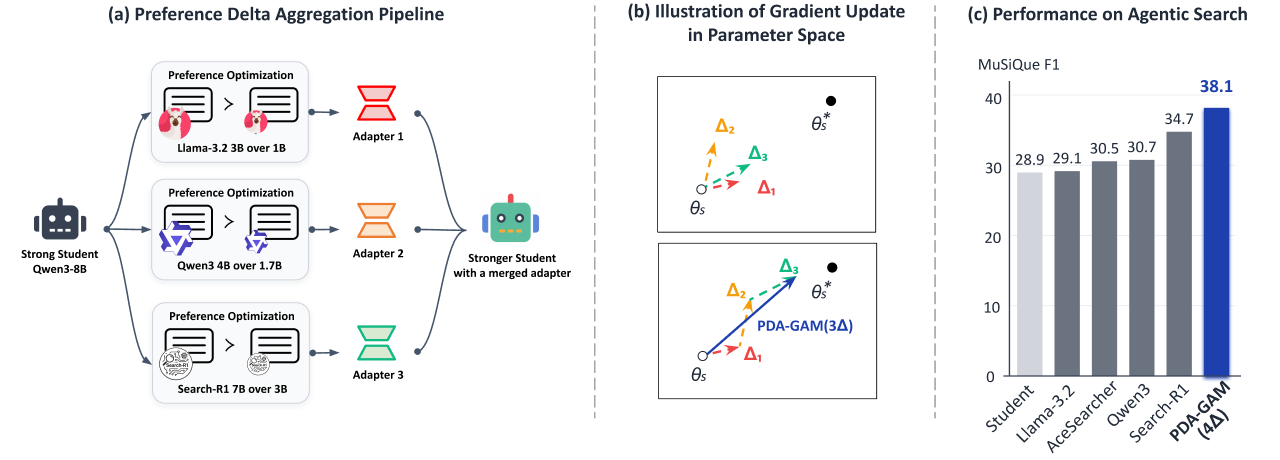}
    \caption{\textbf{(a)} Preference Delta Aggregation (PDA) independently preference-tunes a strong student model on preference datasets from different weak-weaker model pairs using LoRA fine-tuning, then merges the resulting adapters. \textbf{(b)} In parameter space, each adapter induces a distinct update direction, and PDA aggregates these deltas to compose complementary improvements, moving the student toward a better solution. \textbf{(c)} By aggregating multiple preference deltas, PDA achieves larger gains than baselines trained on preference data from any single weak-weaker model pair.}
    \label{fig:figure1}
\end{figure}

Using three reasoning deltas and four agentic deltas, \shortours-\shortmerging improves Qwen3-8B by 6.8 and 7.3 points on average, respectively. It exceeds the strongest single delta by 2.1 and 4.3 points, and similar gains are observed on \tulu-8B. Performance increases along the evaluated sequences of additional deltas. \shortmerging achieves the highest average scores among the compared operators at these configurations, while TIES-Merging performs better with two deltas on agentic search. Behavioral analysis suggests that aggregation combines complementary capabilities learned from different source pairs. Our contributions are as follows:
\begin{enumerate}[
label=\textbf{(\arabic*)},
leftmargin=0pt,
itemindent=2.2em,
labelsep=0em,
align=left,
itemsep=4pt,
topsep=2pt,
parsep=0pt
]
\item \textbf{A framework for improving strong models by aggregating ``weak'' preference signals.}
We propose \shortours, the first framework that improves strong models without relying on high-quality data by leveraging paired preference data from multiple weak-weaker model pairs and aggregating the resulting signals through LoRA merging.

\item \textbf{A LoRA merging method that mitigates directional interference among preference deltas.}
We introduce \shortmerging, which aligns singular bases through orthogonal Procrustes rotations and averages the aligned factors to produce an update of rank at most $r$.

\item \textbf{Strong empirical results on reasoning and agentic search tasks.}
Experiments on two domains and two student models show that \shortours-\shortmerging exceeds the strongest single-delta and training-based aggregation baselines at the largest evaluated delta sets. Additional analyses examine training variance, merge rank, training budget, and preference-label reliability.
\end{enumerate}

\section{Related Work} \label{sec:related-work}

\paragraph{Weak-to-strong Generalization.}

Weak-to-strong generalization studies whether supervision from a weaker model can effectively improve a stronger one without collapsing to the weak supervisor's performance ceiling \citep{burns2023weaktostronggeneralizationelicitingstrong,ye2026weaktostrong}. Prior studies largely center on two directions: uncovering capabilities already present in pretrained base models \citep{unreasonable}, and using models to recursively refine the data used for subsequent training \citep{meta}. In both settings, supervision is commonly delivered as pointwise judgments from a single weak teacher. However, recent work on AI feedback suggests that such absolute supervision can propagate the teacher's systematic errors and hallucination biases \citep{zheng2023judgingllmasajudgemtbenchchatbot,li2024hrlaifimprovementshelpfulnessharmlessness}. This motivates relative supervision signals based on quality differences rather than absolute judgments \citep{weaktostrong,varying,geng2025deltalearninghypothesispreference}.

\paragraph{Delta Learning.}
Existing alignment approaches rely heavily on human annotations or supervision from frontier models through frameworks such as RLHF or DPO \citep{ouyang2022traininglanguagemodelsfollow, bai2022traininghelpfulharmlessassistant, rafailov2024directpreferenceoptimizationlanguage}. 
Delta learning relaxes this costly supervision requirement by constructing preference pairs from a weak model and its weaker variant, using the quality gap between their responses as an optimization signal for a stronger student \citep{geng2025deltalearninghypothesispreference}. 
While delta learning is related to synthetic preference generation \citep{yuan2025selfrewardinglanguagemodels}, its application to multi-turn agentic search remains underexplored. In this setting, models iteratively interact with external tools \citep{yao2023reactsynergizingreasoningacting}. We evaluate delta learning on agentic search and study whether aggregating multiple deltas can improve on the strongest individual delta.

\paragraph{LoRA Tuning and Merging.}
Low-Rank Adaptation (LoRA) \citep{hu2021loralowrankadaptationlarge} enables parameter-efficient adaptation of large language models by introducing lightweight low-rank modules, avoiding full-model fine-tuning for each downstream task. Its modular structure naturally supports post-hoc composition, leading to increasing interest in combining specialized capabilities learned from different tasks or domains \citep{jin2025datalessknowledgefusionmerging}. 
Task Arithmetic \citep{ilharco2023editingmodelstaskarithmetic} composes parameter updates through weighted addition. TIES-Merging \citep{yadav2023tiesmergingresolvinginterferencemerging} and DARE \citep{yu2024languagemodelssupermario} reduce interference through sparsification and, in TIES, sign resolution. KnOTS \citep{stoica2024modelmergingsvdtie} concatenates updates and computes a joint SVD to obtain a shared basis, then applies an existing merging operator to adapter-specific factors. \ourmerging instead decomposes each adapter separately, aligns both singular bases through orthogonal Procrustes, and averages the aligned bases and core matrices. AdapterSoup \citep{adaptersoup} and LoraHub \citep{lorahub} select or weight adapters for transfer across domains or tasks. Our experiments aggregate adapters derived from different source model pairs within each training domain.

\paragraph{Noisy Preference Labels.}
Preference labels produced by models are often noisy or mutually inconsistent. Prior work measures the consistency and transitivity of LLM judges \citep{feng2025rightway}, aggregates the cyclic preferences of multiple evaluators into acyclic rankings \citep{hu2026acyclic}, resolves contradictions within the preferences of a single model \citep{zhang2024contrasolver}, and designs preference objectives that remain stable under label noise \citep{shankar2025selfconcordant}. Delta learning assigns each label by model size, which is a noisy proxy for response quality at the level of individual pairs. We report how often this label agrees with gold answers for every source pair in Appendix~\ref{app:label_audit}.

\section{\ours} \label{sec:pda}

To aggregate multiple ``weak'' supervision signals, we propose \textbf{\ours\ (\shortours)}, a two-stage framework that first derives a preference delta from each weak-weaker model pair and then aggregates the resulting LoRA-instantiated updates through parameter-space merging.

\paragraph{Deriving Preference Deltas.}
Training on paired responses allows a student to learn from relative response quality \citep{geng2025deltalearninghypothesispreference}. We construct $n$ preference datasets from $n$ model families. For each family $i$, a weak model $\pi_{w_i}$ and a smaller model $\pi_{w'_i}$ generate responses to the same queries. We label the larger model's response as chosen and the smaller model's as rejected, yielding $\mathcal{D}_i$. Model size serves as a noisy proxy for pairwise quality. We reserve $\pi$ for the fixed source policies and $\theta$ for the student parameters.

We train the student $\theta_s$ on each dataset $\mathcal{D}_i$ through preference optimization, using objectives such as DPO \citep{rafailov2024directpreferenceoptimizationlanguage} or ORPO \citep{hong2024orpomonolithicpreferenceoptimization}. We call the resulting parameter update a \textit{preference delta}, denoted by $\Delta_i$. Its usefulness for the student is evaluated empirically. This supervision does not require either source response to match the quality of the student. Instead, the preference objective compares the two responses and changes the student according to their assigned relative ordering. Different source families can supply different response patterns on the same queries, giving rise to distinct updates even when the training questions are shared.

\paragraph{Formulating Delta Aggregation.}
Given preference deltas \(\{\Delta_i\}_{i=1}^{n}\), our goal is to effectively aggregate these diverse signals to improve the strong student model $\theta_s$:
\[
\mathcal{A}(\{\Delta_i\}_{i=1}^{n};\theta_s)=\theta_s',
\]
where $\mathcal{A}$ is an aggregation operator and $\theta_s'$ is the resulting model. Training-based approaches either jointly optimize on $\bigcup_{i=1}^{n}\mathcal{D}_i$ or train sequentially on the datasets. Gradient conflicts and forgetting are potential difficulties in such settings \citep{conflicts,catas}, motivating independent training followed by parameter-space aggregation.

\paragraph{Aggregating Deltas via LoRA Merging.}
\shortours instantiates each preference delta as an independently trained LoRA adapter. 
For each preference dataset \(\mathcal{D}_i\), we fine-tune the same base student model \(\theta_s\) with LoRA, obtaining an adapter \(\Delta W_i\) that realizes the abstract preference delta \(\Delta_i\) in parameter space. Loading this adapter into the base model gives
\[
\theta_s^{(i)} = \theta_s \oplus \Delta W_i,
\]
where \(\oplus\) denotes applying the LoRA adapter to the target modules of \(\theta_s\). The learned adapters \(\{\Delta W_i\}_{i=1}^{n}\) are then merged to produce the final model \(\theta_s'\).

For a given adapted weight matrix, each LoRA adapter represents a low-rank update
\[
\Delta W_i = B_iA_i \in \mathbb{R}^{d \times k},
\]
where \(k\) and \(d\) denote the input and output dimensions, 
\(B_i \in \mathbb{R}^{d \times r}\), \(A_i \in \mathbb{R}^{r \times k}\), and \(r \ll \min(d,k)\). 
A standard post-hoc baseline averages the expanded updates. Our \shortours-Averaging baseline truncates this average to rank at most $r$:
\[
\Delta W_{\mathrm{AVG}}=\frac{1}{n}\sum_{i=1}^{n} B_iA_i,
\qquad
\theta_s' = \theta_s \oplus \mathcal{P}_r(\Delta W_{\mathrm{AVG}}).
\]
Here $\mathcal{P}_r$ retains the leading $r$ singular components. The untruncated average can have rank up to $nr$, whereas both \shortours-Averaging and \shortmerging have rank at most $r$. Appendix~\ref{app:ablation} compares against the untruncated average as well.

Weight-space averaging is invariant to the choice of factors for each update. \shortmerging instead combines the factors separately, so their relative basis orientations matter. We align these bases before aggregation while preserving the update represented by each adapter.

\section{\ourmerging} \label{sec:gam}

\textbf{\ourmerging\ (\shortmerging)} decomposes each adapter through SVD, aligns its left and right singular bases to a reference through orthogonal Procrustes, and averages the aligned factors.

\paragraph{Step 1: Decomposition.}
For each trained LoRA adapter, we factor its update $\Delta W_i = B_iA_i \in \mathbb{R}^{d \times k}$ into $r$ singular components:
\begin{equation}
    \Delta W_i = U_i S_i V_i^\top
\end{equation}
where $U_i \in \mathbb{R}^{d \times r}$ and $V_i \in \mathbb{R}^{k \times r}$ have orthonormal columns, and $S_i \in \mathbb{R}^{r \times r}$ is diagonal. This factorization is exact because $\Delta W_i$ has rank at most $r$. It can be computed without forming the dense update. Given thin QR factorizations $B_i = Q_B R_B$ and $A_i^\top = Q_A R_A$, we take the small SVD $R_B R_A^\top = \hat{U} S_i \hat{V}^\top$ and set $U_i = Q_B\hat{U}$ and $V_i = Q_A\hat{V}$.

\paragraph{Step 2: Basis Alignment by Orthogonal Procrustes.}
A subspace admits multiple orthonormal bases related by rotation \citep{edelman1998geometryalgorithmsorthogonalityconstraints}. Independently trained adapters can therefore have different basis orientations, which can cause cancellation when the bases are averaged directly.

We align each adapter to a reference $\Delta W_1$, selected by validation performance, using orthogonal Procrustes \citep{hurleyorthonal}:
\begin{equation}
    R^U_i = \arg\min_{R \in O(r)} \|U_i R - U_1\|_F, \qquad
    R^V_i = \arg\min_{R \in O(r)} \|V_i R - V_1\|_F.
\end{equation}

These rotations minimize basis mismatch to the reference while preserving each adapter's subspaces. The problem has a closed-form solution: if $U_i^\top U_1 = P_i\Sigma_i Q_i^\top$, then $R_i^U = P_i Q_i^\top$, with the right rotation obtained analogously. Alignment changes the coordinates used to represent an update, rather than moving its column or row subspace. Thus, distinct subspaces need not coincide after alignment. We use an unweighted objective so that alignment treats all retained basis vectors equally. To preserve each complete update, we transform its core matrix along with the bases:
\begin{equation} \label{eq:aligned_factors}
    \tilde{U}_i = U_i R^U_i, \qquad
    \tilde{S}_i = (R^U_i)^\top S_i R^V_i, \qquad
    \tilde{V}_i = V_i R^V_i.
\end{equation}
The aligned core $\tilde{S}_i$ is generally not diagonal. This transformation satisfies $\tilde{U}_i\tilde{S}_i\tilde{V}_i^\top = \Delta W_i$. Reference-choice sensitivity is reported in Appendix~\ref{app:reference_sensitivity}.

\paragraph{Step 3: Averaging and Reconstruction.}
We average the aligned bases and core matrices separately:
\begin{equation}
    \bar{U} = \frac{1}{n} \sum_{i=1}^n \tilde{U}_i, \quad \bar{S} = \frac{1}{n} \sum_{i=1}^n \tilde{S}_i, \quad \text{and} \quad \bar{V} = \frac{1}{n} \sum_{i=1}^n \tilde{V}_i.
\end{equation}
Finally, the averaged components are multiplied to reconstruct the final merged weight update: 
\begin{equation}
    \Delta W_{\text{GAM}} = \bar{U} \, \bar{S} \, \bar{V}^\top
\end{equation}
The product has rank at most $r$ and can be stored as one LoRA adapter with $B' = \bar{U}\bar{S}$ and $A' = \bar{V}^\top$. Its adapter parameter count is independent of $n$. Unlike a sum of complete updates, the product of averaged factors also combines factors originating from different adapters. The rotations therefore affect the merged result even though they leave every individual update unchanged. Appendix~\ref{app:gam_math} proves update preservation under rotation and reduction to ordinary weight averaging when all adapters share both left and right subspaces. These properties clarify where \shortmerging can differ from averaging. Whether that difference improves task performance is evaluated through the merging comparisons and controls below.

\section{Experiments} \label{sec:experiments}

We evaluate \ours on knowledge reasoning and agentic search with two student models. Our experiments examine whether individual weak preference signals transfer to a stronger student, whether aggregation improves on the strongest individual signal, and how the merging operator affects the result. We then use training-budget comparisons, repeated training seeds, and fixed-adapter controls to examine the observed gains.

\subsection{Knowledge Reasoning Task Setup} \label{sec:reasoning_exp}

\paragraph{Tasks and Data.} We evaluate \ours on the knowledge reasoning task using four widely used benchmarks: MATH \citep{lightman2023letsverifystepstep}, GPQA \citep{rein2023gpqagraduatelevelgoogleproofqa}, MMLU \citep{hendrycks2021measuringmassivemultitasklanguage}, and GSM8K \citep{cobbe2021trainingverifierssolvemath}. To construct the preference data, we use the full GSM8K training set, consisting of 7,473 queries. For each query, we generate one pair of response trajectories for every weak-weaker model pair. Preference labels are automatically assigned based on model size, with responses from larger (weak) models treated as preferred over those from smaller (weaker) models ($y_{\text{weak}} \succ y_{\text{weaker}}$). This label is a noisy proxy for response quality at the level of individual pairs, and Appendix~\ref{app:label_audit} reports its agreement with gold answers for all seven source pairs.

\paragraph{Model Pairs and Preference Deltas.} We use Qwen3-8B \citep{yang2025qwen3technicalreport} as the primary strong student model, and Tülu3-8B \citep{lambert2025tulu3pushingfrontiers} to evaluate cross-model generalization.
The preference deltas are derived from three pairs of weak and weaker models: Llama-3.2 3B over 1B \citep{grattafiori2024llama3herdmodels}, Qwen3 4B over 1.7B, and DeepSeek-R1 7B over 1.5B \citep{deepseek}. As detailed in Appendix \ref{app:abs_perf}, all selected weak and weaker LLMs consistently underperform the student model, thereby ensuring a setting with limited high-quality supervision.

\paragraph{Training and Merging Details.} We train one LoRA adapter per source pair using Direct Preference Optimization (DPO) \citep{rafailov2024directpreferenceoptimizationlanguage,hu2021loralowrankadaptationlarge}. \shortmerging combines these adapters as described in \Cref{sec:gam}. Training hyperparameters are provided in Appendix~\ref{app:training_details}.

\subsection{Agentic Search Task Setup} \label{sec:agentic_exp}

\paragraph{Tasks and Data.} We evaluate \ours on agentic search using three single-hop QA benchmarks (NQ \citep{kwiatkowski-etal-2019-natural}, TriviaQA \citep{joshi2017triviaqalargescaledistantly}, PopQA \citep{mallen2023trustlanguagemodelsinvestigating}) and four multi-hop QA benchmarks (HotpotQA \citep{yang2018hotpotqadatasetdiverseexplainable}, 2WikiMultiHopQA \citep{ho2020constructingmultihopqadataset}, Bamboogle \citep{press2023measuringnarrowingcompositionalitygap}, MuSiQue \citep{trivedi2022musiquemultihopquestionssinglehop}). Following the same process described in \Cref{sec:reasoning_exp}, the preference data is constructed using weak-weaker model pairs and 8,000 queries randomly sampled from the MuSiQue training set. Unlike the reasoning setup, agentic search trajectories require strict structural formatting. We apply a formatting filter only to the weak model's (chosen) responses, while retaining formatting errors in the weaker model's (rejected) responses to preserve the negative signal.

 \paragraph{Model Pairs and Preference Deltas.} We use the same student models as in \Cref{sec:reasoning_exp}, namely Qwen3-8B and Tülu3-8B. We construct preference deltas from four LLM families: Llama-3.2 3B over 1B, AceSearcher 7B over 3B \citep{xu2025acesearcherbootstrappingreasoningsearch}, Qwen3 4B over 1.7B, and Search-R1 7B over 3B \citep{jin2025searchr1trainingllmsreason}. Llama 3.2 and Qwen3 are prompted using ReAct \citep{yao2023reactsynergizingreasoningacting} and one example to operate as search agents, whereas AceSearcher and Search-R1 are fine-tuned for this role. Similar to the model selection in \Cref{sec:reasoning_exp}, all selected weak and weaker LLMs underperform the student model as shown in Appendix~\ref{app:abs_perf}.

\paragraph{Training and Merging Details.} Agentic search requires interleaved reasoning and retrieval steps. In preliminary experiments, DPO reduces the student's format correctness rate from 0.97 to 0.83. We therefore use ORPO \citep{hong2024orpomonolithicpreferenceoptimization}, whose supervised component preserves formatting while optimizing preferences, restoring the rate to 0.97.
We also train a chosen-only SFT baseline on the Qwen3-4B responses. It underperforms the untuned student (Appendix~\ref{app:sft}), indicating that direct imitation of these responses is insufficient to explain the preference-tuning gains. Other training settings follow the reasoning setup (Appendix~\ref{app:training_details}).

\subsection{Baselines} \label{sec:baselines}

We compare \ours and \ourmerging against three categories of baselines:
\textbf{(1) Student Baseline}: We evaluate the strong student LLMs, \ie Qwen3-8B and Tülu3-8B, without any fine-tuning. 
\textbf{(2) Single Preference Delta Baselines}: We fine-tune the student separately on the preference data from each source pair.
\textbf{(3) Multiple Preference Delta Baselines}: We compare methods that leverage signals from multiple quality and preference deltas, covering both training-based and LoRA-based approaches. For training-based aggregation methods, we consider \textit{Sequential Training}, where the student is fine-tuned consecutively across different deltas, and \textit{Joint Training}, which optimizes the student model at once with a mixture of data from multiple deltas. For LoRA-based aggregation methods within our \ours framework, we compare our proposed \shortmerging with existing LoRA merging methods, including \textit{Naive Averaging} \citep{ilharco2023editingmodelstaskarithmetic} truncated to rank $r$ (\Cref{sec:pda}), \textit{TIES-Merging} \citep{yadav2023tiesmergingresolvinginterferencemerging}, and \textit{KnOTS} \citep{stoica2024modelmergingsvdtie}.

\subsection{Main Results} \label{sec:main_results}

\begin{table*}[t]
    \centering
\caption{Exact Match scores on knowledge reasoning benchmarks. For multi-delta experiments, $k\Delta$ uses the first $k$ preference deltas in the order DeepSeek-R1, Qwen3, Llama-3.2. For example, 2$\Delta$ denotes aggregating DeepSeek-R1 and Qwen3. The best score is highlighted in \textbf{bold} for single-delta and multi-delta baselines.}
    \small
    \setlength{\tabcolsep}{4pt} %
    \resizebox{\textwidth}{!}{%
    \begin{tabular}{@{} l cccc c cccc c @{}}
        \toprule
        \multirow{2}{*}{\textbf{Method}} & \multicolumn{5}{c}{\textbf{Qwen3-8B}}  & \multicolumn{5}{c}{\textbf{Tülu3-8B}}  \\
        \cmidrule(lr){2-6} \cmidrule(lr){7-11}
        & \textbf{MATH} & \textbf{GPQA} & \textbf{MMLU} & \textbf{GSM8K} & \textbf{Avg.}
        & \textbf{MATH} & \textbf{GPQA} & \textbf{MMLU} & \textbf{GSM8K} & \textbf{Avg.}\\
        \midrule
        Student Baseline & 81.5 & 42.1 & 78.5 & 88.0 & 72.5 & 78.2 & 38.5 & 75.2 & 85.1 & 69.3 \\
        \midrule
        \multicolumn{11}{@{}l}{\textit{Single Preference Delta}} \\
        \quad Llama-3.2 (3B over 1B) & 82.8 & 42.6 & 79.1 & 88.5 & 73.3 & 79.5 & 39.2 & 76.0 & 86.0 & 70.2 \\
        \quad Qwen3 (4B over 1.7B) & 84.5 & 44.0 & 80.2 & 90.2 & 74.7 & 81.2 & 40.5 & 77.5 & 88.2 & 71.9 \\
        \quad DeepSeek-R1 (7B over 1.5B) & \textbf{87.2} & \textbf{46.5} & \textbf{81.8} & \textbf{93.4} & \textbf{77.2} & \textbf{84.6} & \textbf{43.1} & \textbf{79.2} & \textbf{91.0} & \textbf{74.5} \\
        \midrule
        \multicolumn{11}{@{}l}{\textit{Multiple Preference Deltas}} \\
        \quad Sequential DPO (2$\Delta$) & 79.2 & 39.8 & 76.9 & 88.2 & 71.0 & 76.3 & 34.6 & 72.3 & 82.2 & 66.4 \\
        \quad Joint DPO (3$\Delta$)  & 85.1 & 44.3 & 80.6 & 89.6 & 74.9 & 82.7 & 40.2 & 78.4 & 89.5 & 72.7 \\
        \quad \shortours-Averaging (3$\Delta$) & 86.8 & 46.2 & 80.5 & 92.5 & 76.5 & 83.5 & 42.8 & 78.4 & 90.1 & 73.7 \\
        \quad \shortours-KnOTS (3$\Delta$) & 87.8 & 47.1 & 81.4 & 93.7 & 77.5 & 85.5 & 44.2 & 79.7 & 91.4 & 75.2 \\
        \quad \shortours-TIES (2$\Delta$) & 87.6 & 46.8 & 80.4 & 93.7 & 77.1 & 85.2 & 43.9 & 79.5 & 91.2 & 75.0 \\
        \quad \shortours-TIES (3$\Delta$) & 88.1 & 47.3 & 81.5 & 93.8 & 77.7 & 85.8 & 44.5 & 79.8 & 91.5 & 75.4 \\
        \quad \shortours-\shortmerging (2$\Delta$) & 88.5 & 47.9 & 82.4 & 94.2 & 78.3 & 86.2 & 44.8 & 80.5 & 92.2 & 75.9 \\
        \quad \shortours-\shortmerging (3$\Delta$) & \textbf{89.6} & \textbf{48.8} & \textbf{83.1} & \textbf{95.5} & \textbf{79.3} & \textbf{87.4} & \textbf{46.0} & \textbf{81.2} & \textbf{93.5} & \textbf{77.0} \\
        \bottomrule
    \end{tabular}%
    }

    \label{tab:general_results}
\end{table*}

\paragraph{Single preference deltas improve both student models.} Fine-tuning Qwen3-8B on the DeepSeek-R1 pair improves its average reasoning score from 72.5 to 77.2 (\Cref{tab:general_results}), consistent with the Delta Learning Hypothesis \citep{geng2025deltalearninghypothesispreference}. On agentic search, the Search-R1 pair raises average F1 from 50.7 to 53.7 (\Cref{tab:main_results}). Each evaluated source pair improves the corresponding student average, with different pairs yielding different gains. These improvements also hold for \tulu-8B, supporting the use of relative preferences from weaker models across both domains.

\paragraph{Parameter-space aggregation improves on the strongest single delta.} At the largest evaluated delta sets, \shortours-\shortmerging reaches 79.3 on reasoning and 58.0 on agentic search for Qwen3-8B, exceeding the strongest single delta by 2.1 and 4.3 points. Joint training falls below that single-delta baseline, and sequential training falls below the untuned student. TIES also improves Qwen3-8B from the strongest single delta's 53.7 to 57.4 on agentic search. Thus, the benefit of aggregating independently trained deltas extends beyond the proposed merging operator.

\paragraph{The gains persist under the reported training-budget controls.}
On agentic search, joint training and \shortours use the same 30,608 available preference pairs, with 227 and 225 total optimization steps, respectively. Both exclude a 5\% validation split from optimization. Joint ORPO reaches 52.8 average F1, below the strongest single delta at 53.7. Extending that single delta to four epochs reaches 54.1 in 236 steps, compared with 58.0 for \shortours-\shortmerging (Appendix~\ref{app:budget}). The extra epochs reuse the same source pair's data, whereas \shortours trains on responses from four source pairs. These comparisons show that the improvement is not reproduced by simply extending the strongest single-delta run. \shortours does train four separate adapters, while retaining only one merged adapter at inference.

Along the strongest-first sequences in the main tables, \shortours-\shortmerging improves as deltas are added. On Qwen3-8B agentic search, its average F1 rises from 54.8 with $2\Delta$ to 57.4 with $3\Delta$ and 58.0 with $4\Delta$. Even the later sources, which are weaker individually, contribute gains in these combinations. Averaging over all subsets of each size gives 52.35, 53.72, 55.73, and 58.00 on agentic search, and 75.07, 77.40, and 79.30 on reasoning (Appendix~\ref{app:subsets}). Thus, the trend extends beyond the strongest-first sequence, although source quality still matters.

\paragraph{The benefit of \ourmerging varies across domains and delta counts.}
On Qwen3-8B reasoning with $3\Delta$, \shortours-\shortmerging reaches 79.3, compared with 77.7 for TIES, 77.5 for KnOTS, and 76.5 for Averaging. Both TIES and KnOTS exceed the strongest single delta at 77.2, while \shortmerging provides the largest additional gain. On agentic search with $4\Delta$, \shortmerging leads TIES by 0.6 points for Qwen3-8B and 0.2 points for \tulu-8B. TIES performs better with $2\Delta$ for both students, scoring 55.5 versus 54.8 and 54.3 versus 53.7, respectively. Appendix~\ref{app:operator_counts} reports all four operators at each delta count for Qwen3-8B agentic search.

\paragraph{Paired training seeds support the operator comparison.}
We retrain each preference delta with three seeds while holding the preference data, hyperparameters, and evaluation sets fixed. Within each seed, \shortmerging and TIES merge the same adapters, so their paired difference compares operators on a common set of trained updates. Across three training seeds for Qwen3-8B, the paired \shortmerging-minus-TIES difference is $+1.43 \pm 0.43$ points on the four-benchmark reasoning average and $+0.67 \pm 0.40$ on MuSiQue, where $\pm$ denotes the sample standard deviation (Appendix~\ref{app:seeds}). The difference is positive for all three seeds in both settings and for each of the four reasoning benchmarks within each seed. On MuSiQue, the individual methods' scores vary by 3.0 to 3.3 points across seeds, while the paired difference has a standard deviation of 0.40. Reporting paired differences separates the operator comparison from variation shared by the two methods. The agentic seed study covers MuSiQue only, so the seven-benchmark averages in \Cref{tab:main_results} are not verified across training seeds. For the fixed-model seven-benchmark comparison, 10,000 paired bootstrap resamples give a 95\% confidence interval of $[+0.08, +1.18]$ for the \shortmerging-minus-TIES difference (Appendix~\ref{app:seeds}).

\paragraph{Fixed-adapter controls examine alignment, magnitude, and rank.}
On a fixed set of four agentic adapters, untruncated weight averaging scores 37.0 on MuSiQue and rank-truncated averaging scores 36.9. Averaging the SVD factors separately without rotation also scores 36.9, although it is a different operator. We then rescale this componentwise average separately for each adapted matrix to match the corresponding \shortmerging update's Frobenius norm. This control reaches 37.18, compared with 38.15 for \shortmerging, a difference of 0.97 with a paired bootstrap 95\% confidence interval of $[+0.26, +1.92]$. This supports a contribution from basis alignment beyond update magnitude (Appendix~\ref{app:ablation}). With the training rank fixed at 64, \shortmerging exceeds rank-matched weight averaging by 1.2 to 1.7 points across merge ranks from 8 to 64. Both methods degrade below rank 32. These results distinguish the choice of merging operator from the output-rank constraint. Whole-model merging on an A800 takes 0.76\,s for TIES, 4.41\,s for factor-based Averaging, and 6.42\,s for factor-based \shortmerging (Appendix~\ref{app:merge_cost}). Appendix~\ref{app:reasoning-cases} and~\ref{app:search-cases} provide qualitative examples.

\begin{table*}[t]
    \centering
\caption{Token-level F1 scores on agentic search benchmarks. For multi-delta experiments, $k\Delta$ uses the first $k$ preference deltas in the order Search-R1, Qwen3, AceSearcher, Llama-3.2. For example, $2\Delta$ denotes aggregating Search-R1 and Qwen3. The best score is highlighted in \textbf{bold} for single-delta and multi-delta baselines.}
    \setlength{\tabcolsep}{3.5pt} %
    \resizebox{\textwidth}{!}{%
    \begin{tabular}{@{} l ccc cccc c ccc cccc c @{}}
        \toprule
        \multirow{3}{*}{\textbf{Method}} &
        \multicolumn{8}{c}{\textbf{Qwen3-8B}} &
        \multicolumn{8}{c}{\textbf{\tulu-8B}} \\
        \cmidrule(lr){2-9} \cmidrule(l){10-17} 
        & \multicolumn{3}{c}{\textbf{Single-hop QA}} & \multicolumn{4}{c}{\textbf{Multi-hop QA}} & \multirow{2}{*}{\textbf{Avg.}}
        & \multicolumn{3}{c}{\textbf{Single-hop QA}} & \multicolumn{4}{c}{\textbf{Multi-hop QA}} & \multirow{2}{*}{\textbf{Avg.}} \\
        \cmidrule(lr){2-4} \cmidrule(lr){5-8} \cmidrule(lr){10-12} \cmidrule(lr){13-16} 
        & \textbf{NQ} & \textbf{TQA} & \textbf{Pop.}
        & \textbf{Hot.} & \textbf{2Wiki} & \textbf{Bam.} & \textbf{MuSi.} & 
        & \textbf{NQ} & \textbf{TQA} & \textbf{Pop.}
        & \textbf{Hot.} & \textbf{2Wiki} & \textbf{Bam.} & \textbf{MuSi.} & \\
        \midrule
        Student Baseline
        & 46.2 & 61.4 & 44.6 & 53.8 & 58.1 & 61.6 & 28.9 & 50.7
        & 45.8 & 60.5 & 43.0 & 52.4 & 55.1 & 58.8 & 25.6 & 48.7 \\
        \midrule
        \multicolumn{17}{@{}l}{\textit{Single Preference Delta}} \\
        \quad Llama-3.2 (3B over 1B) & 46.5 & 61.6 & 44.9 & 53.9 & 58.5 & 62.0 & 29.1 & 50.9 & 46.0 & 60.8 & 43.5 & 52.6 & 55.8 & 59.5 & 26.2 & 49.2 \\
        \quad AceSearcher (7B over 3B) & 47.8 & 62.1 & 46.2 & 54.1 & 60.1 & 62.8 & 30.5 & 51.9 & 46.5 & 61.5 & 45.2 & 53.0 & 57.4 & 61.0 & 27.8 & 50.3 \\
        \quad Qwen3 (4B over 1.7B) & 48.5 & 62.8 & 47.0 & 54.2 & 61.5 & 63.5 & 32.8 & 52.9 & 46.8 & 62.0 & 46.8 & 53.5 & 59.2 & 62.1 & 28.5 & 51.3 \\
        \quad Search-R1 (7B over 3B) & \textbf{49.1} & \textbf{63.2} & \textbf{47.9} & \textbf{54.4} & \textbf{62.6} & \textbf{64.0} & \textbf{34.7} & \textbf{53.7} & \textbf{47.2} & \textbf{62.8} & \textbf{48.7} & \textbf{53.8} & \textbf{61.3} & \textbf{62.9} & \textbf{29.5} & \textbf{52.3} \\
        \midrule
        \multicolumn{17}{@{}l}{\textit{Multiple Preference Deltas}} \\
        \quad Sequential ORPO (2$\Delta$) & 44.6& 58.2& 41.7& 51.4& 52.2& 59.5& 27.6& 47.9 & 42.8& 56.9& 39.8& 51.6& 50.1& 53.7& 21.4& 45.2 \\
        \quad Joint ORPO (4$\Delta$)& 48.7& 62.8& 46.8& 53.6& 62.0& 63.6& 32.0& 52.8 & 47.0& 62.3& 47.1& 53.4& 60.6& 62.5& 28.9& 51.7 \\
        \quad \shortours-Averaging (4$\Delta$)
        & 53.2 & 65.4 & 50.3 & 57.9 & 64.9 & 69.5 & 36.9 & 56.9
        & 51.2 & 65.8 & 52.5 & 56.4 & 63.8 & 68.3 & 32.0 & 55.7 \\
        \quad \shortours-KnOTS (4$\Delta$) & 53.0 & 65.2 & 50.1 & 57.8 & 64.7 & 69.8 & 37.0 & 56.8 & 50.9 & 65.5 & 52.3 & 56.2 & 63.6 & 68.1 & 31.9 & 55.5 \\
        \quad \shortours-TIES (2$\Delta$) & 51.6 & 64.7 & 49.2 & 56.8 & 63.4 & 66.9 & 35.6 & 55.5 & 49.0 & 64.1 & 50.7 & 55.5 & 62.8 & 66.4 & 31.6 & 54.3 \\
        \quad \shortours-TIES (3$\Delta$)
        & 52.4 & 65.0 & 51.2 & 58.2 & 65.2 & 69.7 & 37.8 & 57.1
        & 50.9 & 65.4 & 52.6 & 56.8 & 63.8 & 68.6 & 32.0 & 55.7 \\
        \quad \shortours-TIES (4$\Delta$)
        & 52.8 & 65.3 & 51.5 & 58.8 & 65.4 & 70.3 & 37.8 & 57.4
        & 51.4 & 65.7 & 52.8 & 57.2 & 64.0 & 68.8 & 32.1 & 56.0 \\
        \quad \shortours-\shortmerging  (2$\Delta$)
        & 50.4 & 64.1 & 48.7 & 55.3 & 63.5 & 65.7 & 36.1 & 54.8
        & 49.6 & 64.0 & 50.4 & 54.4 & 62.4 & 64.7 & 30.3 & 53.7 \\
        \quad \shortours-\shortmerging (3$\Delta$)
        & 52.8 & 65.7 & 51.6 & 58.6 & 65.6 & 69.8 & 37.9 & 57.4
        & 51.1 & 65.6 & 52.7 & 57.1 & 63.8 & 68.8 & 32.4 & 55.9 \\
        \quad \shortours-\shortmerging (4$\Delta$)
        & \textbf{53.2} & \textbf{65.9} & \textbf{52.4} & \textbf{59.3} & \textbf{65.9} & \textbf{71.2} & \textbf{38.1} & \textbf{58.0}
        & \textbf{51.6} & \textbf{65.9} & \textbf{52.9} & \textbf{57.2} & \textbf{64.2} & \textbf{69.1} & \textbf{32.8} & \textbf{56.2} \\
        \bottomrule
    \end{tabular}%
    }

    \label{tab:main_results}
\end{table*}

\section{Analysis} \label{sec:analysis}

\subsection{Are preference deltas always effective signals for learning stronger models?} \label{sec:self_play}

The Delta Learning Hypothesis \citep{geng2025deltalearninghypothesispreference} motivates learning from relative quality differences between weak and weaker models. Following this setting, we construct preference pairs from models in the same family that differ in scale. This choice provides a simple label proxy, whose reliability we assess against gold answers.

Across all seven source pairs, chosen responses have higher exact match than rejected responses, by 13.76 to 16.28 points on agentic search and 6.27 to 7.68 on reasoning (Appendix~\ref{app:label_audit}). Exact match distinguishes only 17.6\% to 23.6\% of pairs. Within this subset, size-based labels agree with gold answers on 84.5\% to 89.1\% of agentic pairs and 66.3\% to 66.8\% of reasoning pairs. The positive aggregate difference therefore coexists with noisy individual labels, particularly on reasoning tasks. For most pairs, exact match cannot establish an ordering because both responses are incorrect. Such responses may still differ in intermediate reasoning or search behavior, but the audit does not measure those differences. The reported agreement is therefore conditional on separable pairs and should not be interpreted as the accuracy of all training labels. This distinction matters when using model size to assign preferences without an external judge.

Beyond the setting studied in the Delta Learning Hypothesis, we identify an additional failure case: the preference delta derived from the ZeroSearch family (\ie 7B over 3B) \citep{sun2025zerosearchincentivizesearchcapability} fails to improve a stronger student model. In our experiments, training on this preference data caused a 1.3 F1 score drop on MuSiQue relative to the student baseline (\ie Qwen3-8B). ZeroSearch trains LLM search agents without real search engines by relying on a simulation LLM that generates progressively harder retrieval environments. We hypothesize that the degradation stems from mismatched task difficulties induced by different simulation LLMs during training, causing the two models to adapt along different informative axes despite sharing the same family. A similar failure was observed when experimenting with another self-play framework, SSRL \citep{ssrl}. These results suggest that preference deltas are not universally beneficial and should be selected carefully before aggregation.

\subsection{What capabilities are composed when aggregating preference deltas?}\label{sec:behavioral}

\begin{table*}[t!]
\centering
\caption{Model behavior analysis on 500 MuSiQue dev-set samples, annotated by Qwen3-32B and verified by the authors. Single-delta baselines exhibit distinct behavioral specializations (\eg Search-R1 excels at adaptive search but struggles with verification), whereas \shortours-\shortmerging balances these complementary behavioral capabilities.}
\small
\resizebox{\textwidth}{!}{
\begin{tabular}{lccccc}
\toprule
Method & Steps & Verification (\%) & Adaptive Search (\%) & Authority (\%) & Recovery (\%)\\
\midrule
Student Qwen3-8B        & 4.28 & 57.2 & 23.6 & 0.6 & 9.3\\
\midrule
\multicolumn{4}{l}{\textit{Single Preference Delta}} \\
\quad  Qwen3 (4B over 1.7B)      & 4.64 & 75.9 & 14.7 & 2.6 & 10.7\\
\quad  Search-R1 (7B over 3B)  & 4.09 & 5.3  & 88.9 & 0.8 & 24.3\\
\midrule
\shortours-\shortmerging (2$\Delta$) & 4.43 & 61.2 & 56.7 & 1.2 & 17.4\\
\bottomrule
\end{tabular}
}

\label{tab:behavioral}
\end{table*}

We examine how aggregation affects search behavior in a two-delta comparison between Search-R1, Qwen3, and their \shortours-\shortmerging combination. Prior work identifies Information Verification, Authority Evaluation, Adaptive Search, and Error Recovery as relevant behaviors \citep{jin2026beneficialreasoningbehaviorsagentic}. Table~\ref{tab:behavioral} shows distinct specializations: the Search-R1 delta yields 88.9\% Adaptive Search but 5.3\% Information Verification, whereas the Qwen3 delta yields 75.9\% Information Verification but 14.7\% Adaptive Search.

\shortours-\shortmerging ($2\Delta$) reaches 56.7\% Adaptive Search and 61.2\% Information Verification, exceeding the averages of the two specialists at 51.8\% and 40.6\%, respectively. The merged model does not reach either specialist's peak, but it avoids the low rate of that specialist's other behavior. This is useful in multi-hop search, where an agent must both retrieve relevant evidence and check whether that evidence supports its answer. Improving only one of these behaviors can leave the other as a practical bottleneck. The examples in Appendix~\ref{app:search-cases} illustrate their combination within the same model, offering a possible behavioral explanation for the improved search performance. These behavioral frequencies describe generated trajectories rather than a decomposition of the model's parameters.

\section{Conclusion} \label{sec:conclusion}

We introduce \shortours to aggregate preference signals from multiple weak-weaker model pairs through independent LoRA training and parameter-space merging. Our \shortmerging operator aligns singular bases and averages the aligned factors to combine these updates. With three reasoning deltas and four agentic deltas, \shortours-\shortmerging improves Qwen3-8B by 6.8 and 7.3 points on average, exceeding the strongest single delta by 2.1 and 4.3 points. The results support aggregation of complementary weak supervision beyond individual source pairs in the evaluated settings.

\section{Limitations} \label{sec:limitation}
Our experiments use one preference-data source per domain, GSM8K and MuSiQue, and a limited set of source pairs. All-subset results cover these source pairs on Qwen3-8B, but do not establish gains for arbitrary new sources. Size-based labels remain noisy, and the exact-match audit cannot resolve preferences when both responses are incorrect. The chosen-only SFT control covers the Qwen3 pair. Training-seed comparisons cover all four reasoning benchmarks but only MuSiQue for agentic search. The two algebraic properties of \shortmerging provide no performance guarantee. We lack a theoretical account of the partially aligned regime and a criterion for predicting which deltas will compose. Other training-based aggregation methods, including gradient surgery \citep{conflicts}, may also be effective.

\begin{ack}
We thank NYU High Performance Computing (HPC) and the NYU Shanghai High Performance Computing Center for providing computational resources. We also thank the NYU Shanghai Department of Computer Science, Data Science, and Engineering for supporting this research. This work received no external funding. The authors declare no competing interests.
\end{ack}

\bibliographystyle{plainnat}       
\bibliography{reference}

\clearpage
\appendix

\section{Performance of Student and Supervisor Models}
\label{app:abs_perf}

To provide a comprehensive view of our experimental setup, we report the absolute zero-shot and few-shot performance of all models (both supervisors and students) used in this study. This serves two purposes: first, to confirm that our supervisor models exhibit a measurable performance gap (\ie preference delta) between the weak and weaker variants; second, to verify that the student models (Qwen3-8B and \tulu-8B) consistently outperform their respective supervisors, ensuring a valid weak-to-strong supervision setting.

As shown in \Cref{tab:absolute_performance}, the intended performance hierarchy is consistently observed. For general-purpose model families like Qwen3 and Llama-3.2, we report performance across all eight benchmarks. For domain-specialized families such as DeepSeek-R1 (reasoning-focused) and the agentic-search models Search-R1 and AceSearcher, we report performance only on their respective target domains. This unified absolute performance table provides the necessary context for the relative improvements discussed in the main results (\Cref{tab:main_results}).

\begin{table*}[h]
    \centering
    \caption{Performance of Student and Supervisor Models. We report zero-shot/few-shot metrics across Knowledge Reasoning and Agentic Search benchmarks. The capability hierarchy (Weaker $<$ Weak $<$ Student) is maintained across all model families, validating the data setup for preference extraction and aggregation.}
    \setlength{\tabcolsep}{3.5pt} 
    \resizebox{\textwidth}{!}{%

    \begin{tabular}{@{} l ccccc ccccc @{}}
        \toprule
        \multirow{2}{*}{\textbf{Model}} & \multicolumn{5}{c}{\textbf{Knowledge Reasoning}} & \multicolumn{5}{c}{\textbf{Agentic Search}} \\
        \cmidrule(lr){2-6} \cmidrule(l){7-11}
        & \textbf{MAT.} & \textbf{GPQ.} & \textbf{MML.} & \textbf{GSM.} & \textbf{Avg.} & \textbf{Hot.} & \textbf{2Wiki} & \textbf{Bam.} & \textbf{Mus.} & \textbf{Avg.} \\
        \midrule
        Qwen3-8B & 81.5 & 42.1 & 78.5 & 88.0 & 72.5 & 53.8 & 58.1 & 61.6 & 28.9 & 50.6 \\
        \tulu-8B & 78.2 & 38.5 & 75.2 & 85.1 & 69.3 & 52.4 & 55.1 & 58.8 & 25.6 & 48.0 \\
        \midrule
        DS-R1-Distill-1.5B (Weaker) & 58.4 & 32.1 & 60.5 & 70.2 & 55.3 & -- & -- & -- & -- & -- \\
        DS-R1-Distill-7B (Weak) & 72.5 & 39.4 & 73.1 & 83.6 & 67.2 & -- & -- & -- & -- & -- \\
        \addlinespace[2pt]
        Qwen3-1.7B (Weaker) & 56.2 & 30.5 & 58.3 & 68.4 & 53.4 & 41.2 & 43.5 & 46.8 & 18.2 & 37.4 \\
        Qwen3-4B (Weak) & 68.1 & 36.8 & 67.2 & 78.5 & 62.7 & 48.5 & 51.2 & 53.4 & 23.6 & 44.2 \\
        \addlinespace[2pt]
        Llama-3.2-1B (Weaker) & 48.5 & 26.4 & 52.1 & 62.3 & 47.3 & 36.4 & 39.8 & 42.1 & 15.6 & 33.5 \\
        Llama-3.2-3B (Weak) & 62.3 & 33.5 & 64.8 & 74.1 & 58.7 & 46.2 & 49.5 & 51.8 & 22.3 & 42.5 \\
        Search-R1-3B (Weaker) & -- & -- & -- & -- & -- & 42.5 & 46.1 & 49.3 & 19.5 & 39.4 \\
        Search-R1-7B (Weak) & -- & -- & -- & -- & -- & 50.1 & 54.3 & 57.2 & 26.4 & 47.0 \\
        \addlinespace[2pt]
        AceSearcher-3B (Weaker) & -- & -- & -- & -- & -- & 41.8 & 44.2 & 47.5 & 18.8 & 38.1 \\
        AceSearcher-7B (Weak) & -- & -- & -- & -- & -- & 49.5 & 52.8 & 55.6 & 25.1 & 45.8 \\
        \bottomrule
    \end{tabular}%
    }
    
    \label{tab:absolute_performance}
\end{table*}

\section{Geometry of the Trained Adapters} \label{app:alignment}

We report two geometric quantities measured on the trained adapters of Qwen3-8B, averaged over all pairs of adapters and all 252 adapted matrices (\Cref{tab:geometry}). The cosine between flattened weight updates $\Delta W_i = B_iA_i$ measures whether two adapters change the weights in similar directions. It is invariant to the rotations used by \shortmerging, because these rotations leave each $\Delta W_i$ unchanged (Appendix~\ref{app:gam_math}). The distance of each Procrustes rotation from the identity measures how differently two adapters orient their singular bases, which is the quantity that \shortmerging changes.

\begin{table}[htbp]
    \centering
    \small
    \caption{Geometry of the trained adapters. $\|R - I\|_F / \sqrt{2r}$ is normalized so that it equals 0 for the identity and has expected squared value 1 for a Haar-random orthogonal matrix.}
    \label{tab:geometry}
    \begin{tabular}{@{} l c c @{}}
        \toprule
        \textbf{Quantity} & \textbf{Knowledge Reasoning} & \textbf{Agentic Search} \\
        \midrule
        Mean pairwise cosine between flattened $\Delta W_i$ & 0.463 & 0.245 \\
        $\|R - I\|_F / \sqrt{2r}$ of the Procrustes rotations & 0.514 & $\approx 1.0$ \\
        \bottomrule
    \end{tabular}
\end{table}

The deltas are partially aligned. Their weight updates are far from co-linear but well above the value expected for unrelated directions in this dimension, and they are more similar on knowledge reasoning than on agentic search. The rotations are far from the identity in both domains. On agentic search their distance from the identity is close to that expected for a random rotation, so averaging the singular bases without alignment combines vectors with largely unrelated orientations. On knowledge reasoning the rotations lie about halfway between the identity and a random rotation.

\section{Exact Properties and Geometric Checks of \shortmerging}
\label{app:gam_math}

This appendix states two properties that hold exactly for \shortmerging as defined in \Cref{sec:gam}. They locate where \shortmerging and naive averaging can differ. They do not show that the difference favors \shortmerging, and that evidence rests on the ablation in Appendix~\ref{app:ablation}. We do not model the adapters as noisy copies of one shared update, since the cosines in \Cref{tab:geometry} are far from the co-linear value such a model implies.

\paragraph{Property 1 (rotation leaves each adapter unchanged).}
For every $i$, $\tilde{U}_i \tilde{S}_i \tilde{V}_i^\top = \Delta W_i$. Using \Cref{eq:aligned_factors} and $RR^\top = I$ for any orthogonal $R$,
\begin{equation*}
    \tilde{U}_i \tilde{S}_i \tilde{V}_i^\top
    = U_i R^U_i (R^U_i)^\top S_i R^V_i (R^V_i)^\top V_i^\top
    = U_i S_i V_i^\top = \Delta W_i .
\end{equation*}
\shortmerging and naive averaging therefore combine the same $n$ matrices. Any difference between them comes from how the aligned factors are combined, not from a change in what each adapter represents.

\paragraph{Property 2 (reduction to naive averaging).}
Suppose the left and right subspaces of all adapters coincide, so that $U_i = U_1 Q_i$ and $V_i = V_1 Q'_i$ for orthogonal $Q_i, Q'_i \in O(r)$. Then $\Delta W_{\text{GAM}} = \frac{1}{n}\sum_{i=1}^n \Delta W_i$. The Procrustes problem $\min_{R \in O(r)} \|U_1 Q_i R - U_1\|_F$ attains zero error at $R^U_i = Q_i^\top$, and this solution is unique because $U_1^\top U_i = Q_i$ is nonsingular. Likewise $R^V_i = Q_i'^\top$. Hence $\tilde{U}_i = U_1$, $\tilde{V}_i = V_1$, and $\tilde{S}_i = Q_i S_i Q_i'^\top$ for every $i$. Averaging gives $\bar{U} = U_1$ and $\bar{V} = V_1$, so
\begin{equation*}
    \Delta W_{\text{GAM}}
    = U_1 \Big( \frac{1}{n}\sum_{i=1}^n Q_i S_i Q_i'^\top \Big) V_1^\top
    = \frac{1}{n}\sum_{i=1}^n U_i S_i V_i^\top
    = \frac{1}{n}\sum_{i=1}^n \Delta W_i .
\end{equation*}
In this case the average already has rank at most $r$, so rank-$r$ truncation leaves it unchanged. By the contrapositive, \shortmerging and naive averaging can differ only when the subspaces of the adapters do not coincide. Neither property gives a performance guarantee or an error bound.

\paragraph{Synthetic numerical check.}
We construct four rank-8 updates of size $64 \times 48$ with shared left and right subspaces, independent orthogonal rotations, and singular values sampled uniformly from $[0.5,2.0)$. Using NumPy float64 with random seed 0, we decompose each update by SVD and apply \shortmerging. The relative Frobenius difference $\|\Delta W_{\mathrm{GAM}}-\Delta W_{\mathrm{AVG}}\|_F/\|\Delta W_{\mathrm{AVG}}\|_F$ is $1.78\times10^{-15}$, where $\Delta W_{\mathrm{AVG}}$ is the untruncated weight average. The script \texttt{scripts/verify\_gam\_properties.py} reproduces this check, including the construction and error definition.

\paragraph{The objective of basis alignment.}
For orthonormal-column matrices $U$ and $U_0$, define $\langle X,Y\rangle_F=\operatorname{tr}(X^\top Y)$. The Procrustes step has the equivalent forms
\begin{align*}
    \|UR-U_0\|_F^2 &= 2r-2\langle UR,U_0\rangle_F, \\
    \max_{R\in O(r)}\langle UR,U_0\rangle_F &= \|U^\top U_0\|_*, \\
    \left\|\frac{UR+U_0}{2}\right\|_F^2 &= \frac{r}{2}+\frac{1}{2}\langle UR,U_0\rangle_F,
\end{align*}
where $\|\cdot\|_*$ denotes the nuclear norm. To obtain the maximum, write $U^\top U_0=P\Sigma Q^\top$ and choose $R=PQ^\top$, giving $\operatorname{tr}(\Sigma)$. For two bases, Procrustes therefore maximizes the squared Frobenius norm of their average over orthogonal rotations, reducing cancellation in that basis average. The same statement holds for the right bases. This is a local objective for each basis-to-reference alignment. It does not establish a global optimum for the multi-adapter average or for the final product $\bar U\bar S\bar V^\top$.

\paragraph{A check on the trained agentic adapters.}
\Cref{tab:pair_geometry} reports all six pairs of the four agentic adapters for Qwen3-8B, measured across the 252 adapted matrices. For orthonormal bases, the singular values of $U_i^\top U_j$ are the cosines of their principal angles. We report their mean, $r^{-1}\|U_i^\top U_j\|_*$, and the analogous quantity for $V$. This differs from the mean squared cosine sometimes used to measure subspace overlap. The pairwise averages are 0.280 on the left and 0.991 on the right. The normalized difference between the two merging operators is larger for pairs with lower left-subspace similarity, with Pearson correlation $-0.96$ across the six rows.

\begin{table}[htbp]
    \centering
    \small
    \caption{Geometry and merge differences for the six agentic adapter pairs. The last column compares \shortmerging with untruncated weight averaging. Principal-angle cosines are averaged over the retained directions and adapted matrices.}
    \label{tab:pair_geometry}
    \begin{tabular}{@{} l c c c @{}}
        \toprule
        \textbf{Adapter pair} & \textbf{Mean cosine, $U$} & \textbf{Mean cosine, $V$} & $\frac{\|\Delta W_{\mathrm{GAM}}-\Delta W_{\mathrm{AVG}}\|_F}{\|\Delta W_{\mathrm{AVG}}\|_F}$ \\
        \midrule
        AceSearcher + Llama-3.2 & 0.429 & 0.993 & 0.186 \\
        Search-R1 + Llama-3.2 & 0.279 & 0.991 & 0.306 \\
        Search-R1 + Qwen3 & 0.269 & 0.991 & 0.361 \\
        Search-R1 + AceSearcher & 0.260 & 0.990 & 0.343 \\
        Qwen3 + Llama-3.2 & 0.225 & 0.990 & 0.427 \\
        Qwen3 + AceSearcher & 0.218 & 0.990 & 0.441 \\
        \bottomrule
    \end{tabular}
\end{table}

The relation has one inversion in the ordering, and the six pairs are not independent because each adapter appears in three pairs. We therefore treat this as a descriptive check in the agentic setting. It links subspace geometry to the difference between the operators, but does not show that the difference improves task performance. That comparison is evaluated separately in Appendix~\ref{app:ablation}.

\section{Training Variance and Bootstrap Intervals} \label{app:seeds}

We retrain every preference delta with three training seeds (42, 100, and 31415) while holding the preference data, hyperparameters, and evaluation sets fixed. Decoding is greedy, so the variation reflects training alone. Seed 42 is the run reported in the main tables. Within each seed, \shortmerging and TIES-Merging are applied to the same adapters, so we report the paired difference together with the individual scores (\Cref{tab:seeds}). On agentic search the study covers MuSiQue at $4\Delta$, and on knowledge reasoning it covers the four-benchmark average at $3\Delta$. Both use Qwen3-8B as the student.

\begin{table}[htbp]
    \centering
    \small
    \caption{\shortmerging against TIES-Merging across three training seeds. The last row reports the mean and the standard deviation across seeds.}
    \label{tab:seeds}
    \resizebox{\linewidth}{!}{%
    \begin{tabular}{@{} l ccc @{\hspace{18pt}} ccc @{}}
        \toprule
        & \multicolumn{3}{c}{\textbf{Agentic Search, MuSiQue, $4\Delta$}} & \multicolumn{3}{c}{\textbf{Knowledge Reasoning, Avg., $3\Delta$}} \\
        \cmidrule(lr){2-4} \cmidrule(l){5-7}
        \textbf{Seed} & \textbf{\shortmerging} & \textbf{TIES} & \textbf{Difference} & \textbf{\shortmerging} & \textbf{TIES} & \textbf{Difference} \\
        \midrule
        42 (reported) & 38.1 & 37.8 & $+0.30$ & 79.25 & 77.67 & $+1.58$ \\
        100 & 41.4 & 40.8 & $+0.60$ & 79.33 & 77.55 & $+1.78$ \\
        31415 & 39.6 & 38.5 & $+1.10$ & 80.17 & 79.22 & $+0.95$ \\
        \midrule
        Mean $\pm$ sd & $39.70 \pm 1.65$ & $39.03 \pm 1.57$ & $+0.67 \pm 0.40$ & $79.58 \pm 0.51$ & $78.15 \pm 0.93$ & $+1.43 \pm 0.43$ \\
        \bottomrule
    \end{tabular}%
    }
\end{table}

On knowledge reasoning, the difference is positive on each of the four benchmarks for every seed. The per-benchmark differences for MATH, GPQA, MMLU, and GSM8K are $+1.5$, $+1.5$, $+1.6$, and $+1.7$ for seed 42, $+1.9$, $+1.4$, $+2.3$, and $+1.5$ for seed 100, and $+1.8$, $+0.5$, $+0.1$, and $+1.4$ for seed 31415. On agentic search, the MuSiQue score of each method varies by 3.0 to 3.3 points across seeds, while the paired difference has a standard deviation of 0.40. Most of the seed variation is therefore shared by the two methods built from the same adapters.

\paragraph{Evaluation-set uncertainty.}
We also report paired bootstrap intervals for two fixed-model comparisons with Qwen3-8B (\Cref{tab:bootstrap}). We resample evaluation questions with replacement 10,000 times, keeping the two systems' scores for each question paired, and recompute the difference in average score. The agentic comparison uses 3,500 questions, with 500 from each of seven benchmarks. The rotation comparison uses 500 MuSiQue questions and the norm-matched control defined in Appendix~\ref{app:ablation}. These 95\% confidence intervals quantify uncertainty from the evaluation questions for the fixed trained models. They do not include training-seed variation, which is reported separately in \Cref{tab:seeds}.

\begin{table}[htbp]
    \centering
    \small
    \caption{Paired bootstrap comparisons using 10,000 resamples. Differences and intervals are in F1 points. Both comparisons use four agentic deltas and Qwen3-8B.}
    \label{tab:bootstrap}
    \begin{tabular}{@{} l r c r @{}}
        \toprule
        \textbf{Comparison and evaluation set} & \textbf{Difference} & \textbf{95\% CI} & \textbf{Questions} \\
        \midrule
        \shortmerging vs. TIES, seven-benchmark average & $+0.62$ & $[+0.08, +1.18]$ & 3,500 \\
        \shortmerging vs. norm-matched control, MuSiQue & $+0.97$ & $[+0.26, +1.92]$ & 500 \\
        \bottomrule
    \end{tabular}
\end{table}

\section{Separating Rotation from Truncation} \label{app:ablation}

\Cref{tab:decomposition} compares merging variants on the same four agentic adapters (Qwen3-8B, $4\Delta$, seed 42) using MuSiQue F1. The componentwise average sets every rotation in \shortmerging to the identity. The norm-matched control rescales this average separately for each adapted matrix to match the Frobenius norm of the corresponding \shortmerging update. Componentwise averaging and truncated weight averaging are different operators despite obtaining the same score here. \shortmerging exceeds the norm-matched control by 0.97 points, with a paired bootstrap 95\% confidence interval of $[+0.26, +1.92]$ (\Cref{tab:bootstrap}). Its score of 38.15 was obtained by re-evaluating the model reported as 38.1 in \Cref{tab:main_results}.

\begin{table}[htbp]
    \centering
    \small
    \caption{Comparison of merging variants for four agentic adapters. Changes are relative to the preceding row. MuSiQue F1, Qwen3-8B.}
    \label{tab:decomposition}
    \begin{tabular}{@{} l l c c @{}}
        \toprule
        \textbf{Merge variant} & \textbf{Operation or control} & \textbf{MuSiQue} & \textbf{Change} \\
        \midrule
        Averaging, untruncated & -- & 37.0 & -- \\
        Averaging, truncated to rank $r$ (\shortours-Averaging) & rank-$r$ truncation & 36.9 & $-0.1$ \\
        Componentwise average of $U_i$, $S_i$, $V_i$ & separate averaging, no rotation & 36.9 & $0.0$ \\
        Componentwise average, norm-matched & norm matching & 37.18 & $+0.28$ \\
        \shortmerging & Procrustes rotation & 38.15 & $+0.97$ \\
        \bottomrule
    \end{tabular}
\end{table}

\Cref{tab:rank_sweep} varies the merge rank while fixing the training rank at 64. \shortmerging exceeds rank-matched weight averaging by 1.2 to 1.7 points across the evaluated ranks. Both methods degrade below rank 32. We retain the prespecified merge rank of 64 for the main results.

\begin{table}[htbp]
    \centering
    \small
    \caption{Merge-rank sweep on four agentic adapters (MuSiQue F1, Qwen3-8B). The training rank is 64 in all rows.}
    \label{tab:rank_sweep}
    \begin{tabular}{@{} l c c c @{}}
        \toprule
        \textbf{Merge rank} & \textbf{Averaging (truncated)} & \textbf{\shortmerging} & \textbf{Difference} \\
        \midrule
        Untruncated & 37.0 & -- & -- \\
        64 (reported) & 36.9 & 38.15 & $+1.25$ \\
        32 & 37.1 & 38.8 & $+1.7$ \\
        16 & 33.8 & 35.2 & $+1.4$ \\
        8 & 31.3 & 32.5 & $+1.2$ \\
        \bottomrule
    \end{tabular}
\end{table}

\section{Training Budget} \label{app:budget}

\Cref{tab:budget} reports the agentic training budget for Qwen3-8B. The pair counts include a 5\% validation split, which is excluded from optimization. Joint training and \shortours use the same 30,608 available pairs, with 227 and 225 optimization steps, respectively. Four epochs of the strongest single delta reuse its training partition rather than adding independent data. Its average F1 is 53.7, 53.9, 54.1, and 54.1 after epochs one to four. \shortours trains four independent rank-64 adapters, increasing the total number of fitted parameters while retaining a single adapter at inference.

\begin{table}[htbp]
    \centering
    \small
    \caption{Training budget on agentic search (Qwen3-8B, global batch size 128). Pair counts include the 5\% held out for validation, while optimization steps exclude it.}
    \label{tab:budget}
    \resizebox{\linewidth}{!}{%
    \begin{tabular}{@{} l c c c c @{}}
        \toprule
        \textbf{Method} & \textbf{Available pairs $\times$ epochs} & \textbf{Optimization steps} & \textbf{Adapters trained} & \textbf{Avg. F1} \\
        \midrule
        Best single delta (Search-R1), 1 epoch & 7,997 & 59 & 1 & 53.7 \\
        Best single delta (Search-R1), 4 epochs & 31,988 ($4 \times 7{,}997$) & 236 & 1 & 54.1 \\
        Sequential ORPO ($2\Delta$) & 15,635 & 115 & 1 & 47.9 \\
        Joint ORPO ($4\Delta$) & 30,608 & 227 & 1 & 52.8 \\
        \shortours-\shortmerging ($4\Delta$) & 30,608 & 225 ($59 + 58 + 56 + 52$) & 4 & 58.0 \\
        \bottomrule
    \end{tabular}%
    }
\end{table}

\section{Preference Label Audit} \label{app:label_audit}

\Cref{tab:label_audit} compares the chosen and rejected responses of every source pair against gold answers using exact match. A pair is separable when exactly one of its two responses is correct. The label accuracy is computed only over separable pairs and is not the accuracy over all pairs. On the remaining pairs, both responses are usually incorrect.

\begin{table}[htbp]
    \centering
    \small
    \caption{Preference label audit by exact match for all seven source pairs.}
    \label{tab:label_audit}
    \resizebox{\linewidth}{!}{%
    \begin{tabular}{@{} l l c c c @{}}
        \toprule
        \textbf{Pair (weak / weaker)} & \textbf{Domain} & \textbf{EM, chosen $-$ rejected} & \textbf{Separable pairs} & \textbf{Label accuracy on separable pairs} \\
        \midrule
        Search-R1 7B / 3B & Agentic & $+16.28$ & 21.3\% & 88.2\% \\
        Qwen3 4B / 1.7B & Agentic & $+14.97$ & 19.2\% & 89.0\% \\
        AceSearcher 7B / 3B & Agentic & $+14.35$ & 20.8\% & 84.5\% \\
        Llama-3.2 3B / 1B & Agentic & $+13.76$ & 17.6\% & 89.1\% \\
        DeepSeek-R1 7B / 1.5B & Reasoning & $+7.68$ & 23.6\% & 66.3\% \\
        Qwen3 4B / 1.7B & Reasoning & $+6.83$ & 20.3\% & 66.8\% \\
        Llama-3.2 3B / 1B & Reasoning & $+6.27$ & 18.8\% & 66.7\% \\
        \bottomrule
    \end{tabular}%
    }
\end{table}

\section{All Delta Subsets} \label{app:subsets}

\Cref{tab:subsets_reasoning,tab:subsets_agentic} report all seven nonempty subsets of the reasoning deltas and all fifteen nonempty subsets of the agentic deltas for Qwen3-8B. Each multi-delta subset is merged with \shortmerging. Single deltas and configurations shared with the main tables use the values in \Cref{tab:general_results,tab:main_results}.

Averaging the reported scores over all subsets of each size gives 75.07, 77.40, and 79.30 for one to three reasoning deltas. The corresponding agentic averages are 52.35, 53.72, 55.73, and 58.00 for one to four deltas. The agentic sequence starting with the weakest delta also improves, from 50.9 (D) to 52.2 (C+D), 54.2 (B+C+D), and 58.0 (A+B+C+D). The growth in these experiments therefore does not depend on selecting only the strongest-first sequence.

Every two-delta merge exceeds its strongest constituent, and each larger merge exceeds every proper subset it contains in these tables. Source quality still matters: the weakest two deltas together score below the strongest single delta in both domains, at 76.1 versus 77.2 on reasoning and 52.2 versus 53.7 on agentic search. The best agentic two-delta subset (54.8) also exceeds the weakest three-delta subset (54.2). These comparisons support gains from combining the evaluated sources without implying that delta count alone determines performance.

\begin{table}[htbp]
    \centering
    \small
    \caption{\shortmerging on every subset of reasoning deltas (Exact Match, Qwen3-8B). A is DeepSeek-R1, B is Qwen3, and C is Llama-3.2.}
    \label{tab:subsets_reasoning}
    \begin{tabular}{@{} l ccccc @{}}
        \toprule
        \textbf{Deltas} & \textbf{MATH} & \textbf{GPQA} & \textbf{MMLU} & \textbf{GSM8K} & \textbf{Avg.} \\
        \midrule
        A & 87.2 & 46.5 & 81.8 & 93.4 & 77.2 \\
        B & 84.5 & 44.0 & 80.2 & 90.2 & 74.7 \\
        C & 82.8 & 42.6 & 79.1 & 88.5 & 73.3 \\
        \midrule
        A + B & 88.5 & 47.9 & 82.4 & 94.2 & 78.3 \\
        A + C & 87.8 & 47.1 & 82.4 & 94.0 & 77.8 \\
        B + C & 86.3 & 45.5 & 80.6 & 91.8 & 76.1 \\
        \midrule
        A + B + C & 89.6 & 48.8 & 83.1 & 95.5 & 79.3 \\
        \bottomrule
    \end{tabular}
\end{table}

\begin{table}[htbp]
    \centering
    \small
    \caption{\shortmerging on every subset of agentic deltas (F1, Qwen3-8B). A is Search-R1, B is Qwen3, C is AceSearcher, and D is Llama-3.2.}
    \label{tab:subsets_agentic}
    \begin{tabular}{@{} l ccc cccc c @{}}
        \toprule
        \textbf{Deltas} & \textbf{NQ} & \textbf{TQA} & \textbf{Pop.} & \textbf{Hot.} & \textbf{2Wiki} & \textbf{Bam.} & \textbf{MuSi.} & \textbf{Avg.} \\
        \midrule
        A & 49.1 & 63.2 & 47.9 & 54.4 & 62.6 & 64.0 & 34.7 & 53.7 \\
        B & 48.5 & 62.8 & 47.0 & 54.2 & 61.5 & 63.5 & 32.8 & 52.9 \\
        C & 47.8 & 62.1 & 46.2 & 54.1 & 60.1 & 62.8 & 30.5 & 51.9 \\
        D & 46.5 & 61.6 & 44.9 & 53.9 & 58.5 & 62.0 & 29.1 & 50.9 \\
        \midrule
        A + B & 50.4 & 64.1 & 48.7 & 55.3 & 63.5 & 65.7 & 36.1 & 54.8 \\
        A + C & 50.1 & 63.5 & 48.2 & 55.0 & 62.9 & 65.1 & 35.8 & 54.4 \\
        A + D & 49.7 & 63.2 & 48.1 & 54.5 & 62.6 & 64.4 & 35.2 & 54.0 \\
        B + C & 49.2 & 63.0 & 47.8 & 54.6 & 62.1 & 64.6 & 34.1 & 53.6 \\
        B + D & 49.0 & 62.9 & 47.3 & 54.3 & 61.8 & 64.1 & 33.4 & 53.3 \\
        C + D & 47.1 & 62.4 & 46.6 & 54.1 & 60.6 & 63.6 & 31.2 & 52.2 \\
        \midrule
        A + B + C & 52.8 & 65.7 & 51.6 & 58.6 & 65.6 & 69.8 & 37.9 & 57.4 \\
        A + B + D & 51.7 & 64.9 & 50.1 & 56.5 & 64.7 & 67.4 & 36.8 & 56.0 \\
        A + C + D & 50.9 & 64.3 & 49.6 & 55.8 & 64.0 & 66.2 & 36.3 & 55.3 \\
        B + C + D & 50.6 & 63.4 & 48.3 & 55.1 & 62.6 & 65.2 & 34.3 & 54.2 \\
        \midrule
        A + B + C + D & 53.2 & 65.9 & 52.4 & 59.3 & 65.9 & 71.2 & 38.1 & 58.0 \\
        \bottomrule
    \end{tabular}
\end{table}

\section{Merge Cost} \label{app:merge_cost}

\Cref{tab:merge_cost} reports merge times on an A800 GPU with TF32 disabled, using the median of five runs. The single-projection comparison uses a $4096 \times 12288$ matrix, four adapters, and rank 64. We distinguish the implementation that forms the dense update before SVD from the implementation that operates on LoRA factors. For \shortmerging, the latter uses the QR-based thin SVD described in \Cref{sec:gam}. The dense implementation was used for the reported task results, while the factor implementation provides an algebraically equivalent computation.

The memory statistic is the size of the largest intermediate tensor. It is not total peak GPU memory, which also includes simultaneously resident tensors and implementation workspace. Over all 252 adapted matrices of Qwen3-8B, the measured whole-model merge times are 0.76\,s for TIES-Merging, 4.41\,s for factor-based Averaging, and 6.42\,s for factor-based \shortmerging. These are measurements across the model's actual matrix shapes, rather than extrapolations from the single projection. Merging is performed offline once, and the result is stored as a single rank-64 adapter.

\begin{table}[htbp]
    \centering
    \small
    \caption{Merge cost for one $4096 \times 12288$ projection with four rank-64 adapters. Memory denotes the largest intermediate tensor, not total peak device memory.}
    \label{tab:merge_cost}
    \begin{tabular}{@{} l l r r @{}}
        \toprule
        \textbf{Operator} & \textbf{Computation} & \textbf{Time (ms)} & \textbf{Largest tensor (MB)} \\
        \midrule
        Averaging & Dense update, then SVD & 3882 & 1007 \\
        Averaging & QR of LoRA factors & 19.7 & 17 \\
        TIES-Merging & LoRA factors, no SVD & 3.4 & 17 \\
        \shortmerging & Dense update, then SVD & 16567 & 1007 \\
        \shortmerging & QR of LoRA factors & 27.5 & 17 \\
        \bottomrule
    \end{tabular}
\end{table}

\section{Merging Operators at Each Delta Count} \label{app:operator_counts}

\Cref{tab:operator_counts} compares four merging operators on the same nested subsets of agentic deltas for Qwen3-8B. Deltas are included in the order Search-R1, Qwen3, AceSearcher, and Llama-3.2, matching \Cref{tab:main_results}. \shortmerging exceeds KnOTS by 0.7, 1.6, and 1.2 points with two, three, and four deltas, respectively. TIES-Merging is strongest with two deltas, while \shortmerging is strongest with three and four. This comparison describes the operators' task performance without attributing the differences to any single internal operation.

\begin{table}[htbp]
    \centering
    \small
    \caption{Seven-benchmark average F1 for agentic search with Qwen3-8B. All operators merge the same adapters at each delta count. Averaging is truncated to rank 64.}
    \label{tab:operator_counts}
    \begin{tabular}{@{} l ccc @{}}
        \toprule
        \textbf{Operator} & $2\Delta$ & $3\Delta$ & $4\Delta$ \\
        \midrule
        \shortours-Averaging & 54.5 & 56.4 & 56.9 \\
        \shortours-KnOTS & 54.1 & 55.8 & 56.8 \\
        \shortours-TIES & \textbf{55.5} & 57.1 & 57.4 \\
        \shortours-\shortmerging & 54.8 & \textbf{57.4} & \textbf{58.0} \\
        \bottomrule
    \end{tabular}
\end{table}

\section{Experimental and Training Details}
\label{app:training_details}

In this section, we outline the hyperparameters, environmental setups, and data processing configurations used to extract preference deltas and evaluate the aggregated models. All training procedures are conducted on 4 NVIDIA H20 (96GB) GPUs.

\subsection{Optimization and Hyperparameters}
To parameterize the preference deltas efficiently, we apply Low-Rank Adaptation (LoRA) across all base models. The LoRA modules are systematically attached to the attention projection matrices (\texttt{q\_proj}, \texttt{k\_proj}, \texttt{v\_proj}, \texttt{o\_proj}) and MLP layers (\texttt{gate\_proj}, \texttt{up\_proj}, \texttt{down\_proj}) across all transformer blocks. 

To improve clarity and reproducibility, the detailed optimization settings and hyperparameter configurations shared across both task domains are summarized in Table~\ref{tab:hyperparameters}. Gradient checkpointing is enabled in all runs to optimize memory consumption.

\begin{table}[htbp]
    \centering
    \small
    \caption{Hyperparameters for Preference Delta Extraction. These configurations are consistently applied across both Knowledge Reasoning and Agentic Search domains unless otherwise specified.}
    \begin{tabular}{@{} l c @{}}
        \toprule
        \textbf{Hyperparameter} & \textbf{Value} \\
        \midrule
        LoRA Rank ($r$) & 64 \\
        LoRA Alpha ($\alpha$) & 128 \\
        LoRA Dropout & 0.05 \\
        Target Modules & \texttt{q, k, v, o, gate, up, down\_proj} \\
        Optimizer & AdamW \\
        Peak Learning Rate & $1.0 \times 10^{-5}$ \\
        Learning Rate Scheduler & Cosine \\
        Warmup Ratio & 10\% \\
        Global Batch Size & 128 \\
        Training Epochs & 1 \\
        Max Sequence Length & 8192 \\
        Preference Objective (Reasoning) & DPO ($\beta = 0.1$) \\
        Preference Objective (Agentic) & ORPO ($\beta = 0.1$) \\
        \bottomrule
    \end{tabular}
    \label{tab:hyperparameters}
\end{table}

\subsection{Knowledge Reasoning Details (DPO Setup)}
For the knowledge reasoning domain, the base models are fine-tuned via Direct Preference Optimization (DPO). The preference datasets are sampled from the training splits of the respective mathematical and reasoning benchmarks. During DPO tuning, the prompts are formatted strictly according to the default chat templates of the respective base models (e.g., Qwen and Llama). To encourage explicit logical derivation, a standard system prompt is utilized to instruct the model to provide step-by-step reasoning (e.g., ``Let's think step by step'') before formulating the final answer. This setup maintains a consistent structural format for reasoning trajectories across both the supervisor and student models.

\subsection{Agentic Search and Retrieval Details (ORPO Setup)}

The preference data is constructed using weak-weaker model pairs and 8,000 queries randomly sampled from the MuSiQue training set. Unlike the reasoning setup, agentic search trajectories require strict structural formatting. We apply a formatting filter only to the weak model's (chosen) responses, while retaining formatting errors in the weaker model's (rejected) responses to preserve the negative signal. As a result, the number of valid preference pairs varies: 7,997 for Search-R1, 7,852 for AceSearcher, 7,638 for Qwen3, and 7,121 for Llama-3.2.

Training and evaluating models for agentic search involves multi-turn interactions where the model iteratively issues search queries and processes external context to answer complex questions (e.g., HotpotQA, MuSiQue). For this domain, we utilize Odds Ratio Preference Optimization (ORPO) to construct the preference deltas.

\paragraph{Agentic Data Formulation.}
The agentic preference pairs are formatted as ReAct trajectories \citep{yao2023reactsynergizingreasoningacting}, comprising interleaved \texttt{Thought}, \texttt{Action}, and \texttt{Observation} steps. Models are explicitly optimized to output targeted action commands (e.g., \texttt{<search>query</search>}) to trigger the external retrieval tool during complex reasoning paths. The sequence length is set to 8192 tokens to accommodate the extensive context windows required by multi-hop retrieval tasks.

\paragraph{Retrieval Environment.}
To ensure reproducible and standardized evaluations, the external search environment is grounded in a fixed Wikipedia corpus (using the \texttt{enwiki-20200601} English Wikipedia dump). When a model generates a \texttt{Search} action, the environment queries the corpus using a dense retrieval model (\texttt{e5-large-v2}), which returns the top $k = 5$ most relevant passages based on embedding similarity. To manage context density and mitigate noise, each retrieved snippet is truncated to a maximum of 500 tokens before being appended to the model's context as an \texttt{Observation}. The model subsequently processes this new information to generate the next \texttt{Thought} or conclude with a final \texttt{<answer>answer</answer>}. This retrieval protocol is maintained uniformly across both the preference data generation and evaluation phases.

\section{Prompt Templates}
\label{app:prompts}

In this section, we provide the exact prompt templates utilized during the preference data extraction and evaluation phases. To standardize the generation distribution and ensure strict adherence to task-specific formats, we employ a one-shot prompting strategy across both domains. The one-shot examples explicitly demonstrate the expected reasoning traces and tool-use trajectories.

\subsection{Knowledge Reasoning Prompts}
For mathematical and logical reasoning tasks, the one-shot prompt instructs the model to articulate its derivation process before formulating the final answer. This structural constraint facilitates the generation of detailed step-by-step preference trajectories for DPO tuning.

\vspace{0.5em}
\noindent\textbf{System Prompt:}
\begin{quote}
\small\ttfamily
You are an expert mathematical and logical reasoning assistant. Please think step by step to solve the problem and carefully explain your derivation before providing the final answer.
\end{quote}

\noindent\textbf{User Input Template (One-shot):}
\begin{quote}
\small\ttfamily
Here is an example of how to answer a question:

Question: A train travels at a constant speed of 60 miles per hour. How far will it travel in 2.5 hours?
Answer: Let's think step by step. We know the formula for distance is Distance = Speed * Time. The given speed is 60 miles per hour, and the given time is 2.5 hours. Multiplying these values: 60 * 2.5 = 150. Therefore, the train will travel 150 miles. The final answer is 150.

Now, answer the following question:

Question: \{question\}
Answer:
\end{quote}

\subsection{Agentic Search Prompts}
For the agentic search domain, models are prompted to interact with the external dense retrieval environment (\texttt{e5-large-v2}) following the ReAct framework \citep{yao2023reactsynergizingreasoningacting}. The one-shot example explicitly dictates the strict interleaving format of thoughts, search actions, and observations.

\vspace{0.5em}
\noindent\textbf{System Prompt (ReAct Instructions):}
\begin{quote}
\small\ttfamily
Answer the following question from the user with the help of a Wikipedia search engine. Please reason step by step. You should think about what you need to know in order to answer the question, and then search for that information using the search engine. To perform a search operation, write a web search question and enclose it with <search> and </search>. You will immediately observe a piece of search results within the <information> and </information> tags. You can then use this retrieved information to continue your reasoning. You can repeat the search process many times. Once you think you have all the information you need, you can end the thinking process and provide the final answer. You MUST enclose your final answer with <answer> and </answer>. 
\end{quote}

\noindent\textbf{User Input Template (One-shot):}
\begin{quote}
\small\ttfamily
Here is an example of the interaction format:

<search> Who were the people that captured Malakoff? </search>
<information> The French army under General MacMahon successfully captured the Malakoff redoubt on 8th. </information>
  
Okay, so the French people captured Malakoff. Now, the next step would be to figure out in what region Pilipsburg is located. I will write a web search to look that up.
  
<search> Where is Philipsburg located at? </search>
<information> Philipsburg is is the main town and capital of Sint Maarten, a constituent country of the Kingdom of the Netherlands. </information>
  
...[more thoughts shortened]...

Your final response:
<answer> November 12, 1625 </answer>
\end{quote}

\section{Case Studies: Knowledge Reasoning}
\label{app:reasoning-cases}

We present four GPQA reasoning traces in which the model trained on the strongest individual preference pair (\textbf{Single Preference Delta}) fails, whereas our proposed \textbf{\shortours-\shortmerging} successfully derives the correct answer.

The failures of the single-delta model typically fall into two categories: \emph{generation non-termination}, where the model enters a degenerate repetitive loop (Table~\ref{tab:rcase-loop}); and \emph{logical propagation errors}, where it produces an incorrect final answer due to intermediate arithmetic or domain knowledge mistakes (Tables~\ref{tab:rcase-arith}--\ref{tab:rcase-chem}).
Across a random sample of 60 GPQA error cases from the single-delta baseline, 35 cases exhibit format/termination failures and 25 exhibit logical errors; \shortours-\shortmerging effectively recovers all 60 cases.
Key reasoning steps are highlighted in \colorbox{gray!12}{\strut gray}, with critical errors in \colorbox{red!12}{\strut red} and the recovered correct derivations in \colorbox{green!12}{\strut green}.

\begin{table}[htbp]
\centering
\renewcommand{\arraystretch}{1.2}
\caption{\textbf{Termination Recovery:} The Single Preference Delta model derives the correct epistasis hierarchy but enters a degenerate repetition loop\textemdash{}repeating the same hesitation 59 times\textemdash{}and fails to emit a valid \texttt{ANSWER:} tag. \shortours-\shortmerging reaches the same conclusion and terminates cleanly.}
\label{tab:rcase-loop}
\vspace{0.3em}
\footnotesize
\begin{tabular}{p{0.95\textwidth}}
\toprule
\textbf{Question:} You perform a high-throughput experiment on white lupine to find genes contributing to resistance to anthracnose. You receive candidate genes G1, G2, G3 and create knockouts... Which statement correctly describes the epistatic relationships? \\
\textbf{Gold:} A \\
\midrule
\textbf{Single Preference Delta} \quad (\ding{55} \textit{Output:} None / Format Failure) \\[3pt]
\colorbox{gray!12}{Identifies hierarchy: G2 is epistatic to G1; G2 is epistatic to G3; G1 is epistatic to G3.} \\[2pt]
\colorbox{gray!12}{``So, gene 2 is the most upstream, then gene 1, then gene 3.''} \quad (\textit{Correct conclusion}) \\[2pt]
\colorbox{red!12}{``Wait, perhaps the correct answer is option A... But the data doesn't support that.''} \\[2pt]
$\hookrightarrow$ \textit{Repeated 59$\times$ until token limit; no answer extracted.} \\
\midrule
\textbf{\shortours-\shortmerging} \quad (\ding{51} \textit{Output:} A) \\[3pt]
\colorbox{gray!12}{Derives same hierarchy: G2 $\rightarrow$ G1 $\rightarrow$ G3.} \\[2pt]
\colorbox{green!12}{Identifies G2 as the transcription factor and G1/G3 as redundant. Terminates with \texttt{ANSWER: A}.} \\
\bottomrule
\end{tabular}
\end{table}

\begin{table}[htbp]
\centering
\renewcommand{\arraystretch}{1.2}
\caption{\textbf{Arithmetic Recovery:} The Single Preference Delta model formulates the correct decay equation but makes a magnitude error during division, obtaining $10^{-25}$ instead of $10^{-24}$. \shortours-\shortmerging executes the arithmetic correctly.}
\label{tab:rcase-arith}
\vspace{0.3em}
\footnotesize
\begin{tabular}{p{0.95\textwidth}}
\toprule
\textbf{Question:} X is a meson resonance. What is the mean decay distance? Given: $E_X = 8$\,GeV, $m_X = 1.2$\,GeV, $\Gamma_X = 320$\,MeV. \\
\textbf{Gold:} A \\
\midrule
\textbf{Single Preference Delta} \quad (\ding{55} \textit{Output:} C) \\[3pt]
\colorbox{gray!12}{Step 1: $\gamma = E_X / m_X \approx 6.667$} \\[2pt]
\colorbox{red!12}{Step 2: $\tau = \hbar / \Gamma_X = 6.582 \times 10^{-25} / 0.32 \approx 2.057 \times 10^{-25}$\,s} \quad (\textit{Should be $10^{-24}$}) \\[2pt]
\colorbox{gray!12}{Step 3: $L = \gamma c \tau \approx 4.11 \times 10^{-16}$\,m} $\rightarrow$ \textit{Selects C, off by exactly $10\times$.} \\
\midrule
\textbf{\shortours-\shortmerging} \quad (\ding{51} \textit{Output:} A) \\[3pt]
\colorbox{green!12}{$\tau = \hbar / \Gamma_X = 6.582 \times 10^{-25} / 0.32 = 2.057 \times 10^{-24}$\,s} \quad (\textit{Correct exponent}) \\[2pt]
\colorbox{gray!12}{$d = (p / m_X) \cdot c \cdot \tau \approx 4.07 \times 10^{-15}$\,m} $\rightarrow$ \textit{Selects A.} \\
\bottomrule
\end{tabular}
\end{table}

\clearpage %

\begin{table}[htbp]
\centering
\renewcommand{\arraystretch}{1.2}
\caption{\textbf{Algebraic Recovery:} The Single Preference Delta model drops a logarithmic factor when propagating stellar abundance ratios through solar reference values. \shortours-\shortmerging successfully tracks the transformation.}
\label{tab:rcase-astro}
\vspace{0.3em}
\footnotesize
\begin{tabular}{p{0.95\textwidth}}
\toprule
\textbf{Question:} Two stars with $[\text{Si/Fe}]_1 = 0.3$, $[\text{Mg/Si}]_2 = 0.3$... Calculate the ratio $n_{\text{Si},1} / n_{\text{Si},2}$. \\
\textbf{Gold:} C \\
\midrule
\textbf{Single Preference Delta} \quad (\ding{55} \textit{Output:} D) \\[3pt]
\colorbox{gray!12}{Star\_1: $[\text{Si/Fe}]_1 = 0.3 \Rightarrow n_{\text{Si},1} = 10^{0.3} \cdot n_{\text{Si},\odot}$} \\[2pt]
\colorbox{red!12}{Star\_2: Conflates intermediate terms, dropping a factor from $\log_{10}(n_{\text{Si},\odot}/n_\text{H})$.} \\[2pt]
$\Rightarrow$ Yields ratio $\approx 4.0$ (Selects D). \\
\midrule
\textbf{\shortours-\shortmerging} \quad (\ding{51} \textit{Output:} C) \\[3pt]
\colorbox{green!12}{Correctly tracks each bracket transformation through solar reference values.} \\[2pt]
$\Rightarrow$ Obtains ratio $\approx 12.6$ (Selects C). \\
\bottomrule
\end{tabular}
\end{table}

\begin{table}[htbp]
\centering
\renewcommand{\arraystretch}{1.2}
\caption{\textbf{Domain Knowledge Recovery:} The Single Preference Delta model misidentifies the extent of reduction by red P + HI, treating it as a partial transformation. \shortours-\shortmerging applies the correct exhaustive reduction mechanism.}
\label{tab:rcase-chem}
\vspace{0.3em}
\footnotesize
\begin{tabular}{p{0.95\textwidth}}
\toprule
\textbf{Question:} What is the index of hydrogen deficiency (IHD) of the product obtained when 2-formyl-5-vinylcyclohex-3-enecarboxylic acid is reacted with red phosphorus and excess HI? \\
\textbf{Gold:} B \\
\midrule
\textbf{Single Preference Delta} \quad (\ding{55} \textit{Output:} D) \\[3pt]
\colorbox{gray!12}{Identifies starting IHD $= 5$ (1 ring + 2 C=C + 2 C=O).} \\[2pt]
\colorbox{red!12}{``Red P + HI converts -COOH to -COI. Double bonds remain unchanged.''} \\[2pt]
$\Rightarrow$ Selects D (IHD $= 3$). \\
\midrule
\textbf{\shortours-\shortmerging} \quad (\ding{51} \textit{Output:} B) \\[3pt]
\colorbox{green!12}{``Red P + HI is an exhaustive reduction: reduces all C=C and C=O groups.''} \\[2pt]
\colorbox{gray!12}{Product: fully saturated cyclohexane ring with alkyl substituents. IHD $= 1$.} \\
\bottomrule
\end{tabular}
\end{table}

\clearpage

\section{Case Studies: Agentic Search}
\label{app:search-cases}

To illustrate the behavioral complementation achieved by multiple preference deltas, we present four representative agentic search trajectories. Each case highlights a distinct failure mode overcome by our framework. Search queries are in \colorbox{green!15}{\strut green} and retrieved info in \colorbox{blue!10}{\strut blue}.

\begin{table}[htbp]
\centering
\renewcommand{\arraystretch}{1.2}
\caption{\textbf{Entity Disambiguation:} The Single Preference Delta model misidentifies the lead actor as a child actor and issues repetitive queries before generating an unsupported entity. \shortours-\shortmerging formulates a holistic query to resolve the entity accurately.}
\label{tab:case-entity}
\vspace{0.3em}
\footnotesize
\begin{tabular}{p{0.95\textwidth}}
\toprule
\textbf{Question:} Who did the leading actor in \textit{Sammy Going South} play in \textit{The Ten Commandments}? \\
\textbf{Gold:} Dathan \\
\midrule
\textbf{Single Preference Delta} \quad (\ding{55} \textit{Output:} Mered, 6 turns) \\[3pt]
\colorbox{green!15}{\texttt{<search>} Who is the leading actor in Sammy Going South? \texttt{</search>}} \\
\colorbox{blue!10}{\texttt{<information>} Fergus McClelland is the lead actor... \texttt{</information>}} \\[2pt]
\colorbox{red!12}{Repeatedly queries Fergus McClelland's role in The Ten Commandments (which does not exist).} \\
\textcolor{red}{\texttt{<answer>} Mered \texttt{</answer>}} \quad (\textit{Unsupported entity generation}) \\
\midrule
\textbf{\shortours-\shortmerging} \quad (\ding{51} \textit{Output:} Dathan, 2 turns) \\[3pt]
\colorbox{green!15}{\texttt{<search>} Who did the leading actor in Sammy Going South play in the ten commandments? \texttt{</search>}} \\
\colorbox{blue!10}{\texttt{<information>} Edward G. Robinson starred as Dathan in The Ten Commandments... \texttt{</information>}} \\[2pt]
\textcolor{teal}{\texttt{<answer>} Dathan \texttt{</answer>}} \\
\bottomrule
\end{tabular}
\end{table}

\begin{table}[htbp]
\centering
\renewcommand{\arraystretch}{1.2}
\caption{\textbf{Anchoring Bias Recovery:} The Single Preference Delta model exhibits anchoring bias, fixating on a 2004 film adaptation. \shortours-\shortmerging redirects its search to the 2017 series adaptation to resolve the query.}
\label{tab:case-recovery}
\vspace{0.3em}
\footnotesize
\begin{tabular}{p{0.95\textwidth}}
\toprule
\textbf{Question:} The actor who plays Count Olaf in ``A Series of Unfortunate Events'' also plays which character in \textit{Batman Under the Red Hood}? \\
\textbf{Gold:} Nightwing / Dick Grayson \\
\midrule
\textbf{Single Preference Delta} \quad (\ding{55} \textit{Output:} The Joker, 6 turns) \\[3pt]
\colorbox{blue!10}{\texttt{<information>} Jim Carrey played Count Olaf in the 2004 film... \texttt{</information>}} \\[2pt]
\colorbox{red!12}{Model repeatedly searches for Jim Carrey's role in Batman, failing to consider other adaptations.} \\
\textcolor{red}{\texttt{<answer>} The Joker \texttt{</answer>}} \quad (\textit{Hallucination due to anchoring}) \\
\midrule
\textbf{\shortours-\shortmerging} \quad (\ding{51} \textit{Output:} Dick Grayson, 4 turns) \\[3pt]
\colorbox{green!15}{\texttt{<search>} Which character does Neil Patrick Harris play in Batman Under the Red Hood? \texttt{</search>}} \\
\colorbox{blue!10}{\texttt{<information>} Neil Patrick Harris voiced Dick Grayson/Nightwing... \texttt{</information>}} \\[2pt]
\textcolor{teal}{\texttt{<answer>} Dick Grayson/Nightwing \texttt{</answer>}} \\
\bottomrule
\end{tabular}
\end{table}

\clearpage

\begin{table}[htbp]
\centering
\renewcommand{\arraystretch}{1.2}
\caption{\textbf{Multi-hop Chain Completion:} The Single Preference Delta model prematurely terminates the multi-hop chain at Indonesia. \shortours-\shortmerging correctly parses the full relational chain targeting Timor-Leste.}
\label{tab:case-multihop}
\vspace{0.3em}
\footnotesize
\begin{tabular}{p{0.95\textwidth}}
\toprule
\textbf{Question:} Who is the president of the newly declared independent country of the country of the birthplace of Mulham Arufin? \\
\textbf{Gold:} Francisco Guterres \\
\midrule
\textbf{Single Preference Delta} \quad (\ding{55} \textit{Output:} Joko Widodo, 3 turns) \\[3pt]
\colorbox{blue!10}{\texttt{<information>} Mulham Arufin was born in Indonesia... \texttt{</information>}} \\[2pt]
\colorbox{red!12}{Queries for the president of Indonesia, failing to resolve the ``newly declared independent'' clause.} \\
\textcolor{red}{\texttt{<answer>} Joko Widodo \texttt{</answer>}} \quad (\textit{Premature termination}) \\
\midrule
\textbf{\shortours-\shortmerging} \quad (\ding{51} \textit{Output:} Francisco Guterres, 5 turns) \\[3pt]
\colorbox{green!15}{\texttt{<search>} Who is the president of the newly declared independent country of Timor-Leste? \texttt{</search>}} \\
\colorbox{blue!10}{\texttt{<information>} Francisco Guterres is the president of Timor-Leste... \texttt{</information>}} \\[2pt]
\textcolor{teal}{\texttt{<answer>} Francisco Guterres \texttt{</answer>}} \\
\bottomrule
\end{tabular}
\end{table}

\begin{table}[htbp]
\centering
\renewcommand{\arraystretch}{1.2}
\caption{\textbf{Resolution Matching:} The Single Preference Delta model answers at an incorrect granularity (State instead of County). \shortours-\shortmerging identifies that the retrieved city requires an additional bridging query to map to the target county.}
\label{tab:case-granularity}
\vspace{0.3em}
\footnotesize
\begin{tabular}{p{0.95\textwidth}}
\toprule
\textbf{Question:} In which county was Larkin I. Smith born? \\
\textbf{Gold:} Pearl River County \\
\midrule
\textbf{Single Preference Delta} \quad (\ding{55} \textit{Output:} Mississippi, 3 turns) \\[3pt]
\colorbox{blue!10}{\texttt{<information>} Larkin I. Smith was a U.S. Representative from Mississippi... \texttt{</information>}} \\[2pt]
\colorbox{red!12}{Fails to resolve the specific county from general state-level information.} \\
\textcolor{red}{\texttt{<answer>} Mississippi \texttt{</answer>}} \quad (\textit{Resolution mismatch}) \\
\midrule
\textbf{\shortours-\shortmerging} \quad (\ding{51} \textit{Output:} Pearl River County, 3 turns) \\[3pt]
\colorbox{blue!10}{\texttt{<information>} Smith was born in Poplarville, Mississippi... \texttt{</information>}} \\[2pt]
\colorbox{green!15}{\texttt{<search>} In which county is Poplarville, Mississippi? \texttt{</search>}} \quad (\textit{Bridging query}) \\
\colorbox{blue!10}{\texttt{<information>} Poplarville is located in Pearl River County... \texttt{</information>}} \\[2pt]
\textcolor{teal}{\texttt{<answer>} Pearl River County \texttt{</answer>}} \\
\bottomrule
\end{tabular}
\end{table}

\section{Sensitivity Analysis of Reference Choice in \shortmerging}
\label{app:reference_sensitivity}

In Section 4, we introduce \ourmerging (\shortmerging), which utilizes orthogonal Procrustes alignment to project multiple LoRA adapters onto a shared reference coordinate system ($\phi_1$). While the Procrustes formulation is defined relative to a chosen base, we hypothesized that the specific choice of the reference adapter does not significantly impact the final aggregated performance. To empirically verify this, we conduct a sensitivity analysis for both the Knowledge Reasoning and Agentic Search domains by systematically rotating the reference role among all available domain-specific adapters.

\paragraph{Knowledge Reasoning.} 
For the reasoning domain, we aggregate three distinct preference deltas. We evaluate the aggregated performance of the student model (e.g., Qwen3-8B) by iteratively setting each of the three adapters ($\Delta_{\text{DS-R1}}$, $\Delta_{\text{Qwen3}}$, and $\Delta_{\text{Llama}}$) as the reference base $\phi_1$. As shown in Table~\ref{tab:sensitivity_kr}, the maximum fluctuation in the average reasoning score is minimal (within 0.1 points).

\begin{table}[htbp]
    \centering
    \small
    \caption{Sensitivity to Reference Choice on Knowledge Reasoning. Performance of the aggregated model when different preference deltas are selected as the reference base ($\phi_1$).}
    \begin{tabular}{@{} l c @{}}
        \toprule
        \textbf{Reference Choice ($\phi_1$)} & \textbf{Knowledge Reasoning Avg.} \\
        \midrule
        $\Delta_{\text{DS-R1}}$ (DeepSeek-R1 preference delta) & 79.3 \\
        $\Delta_{\text{Qwen3}}$ (Qwen3 preference delta)       & 79.2 \\
        $\Delta_{\text{Llama}}$ (Llama-3.2 preference delta)   & 79.3 \\
        \midrule
        \textbf{Max Fluctuation} & \textbf{$\le$ 0.1} \\
        \bottomrule
    \end{tabular}
    \label{tab:sensitivity_kr}
\end{table}

\paragraph{Agentic Search.}
Similarly, for the agentic search domain, we aggregate four distinct preference deltas. We repeat the sensitivity analysis by rotating the reference role among $\Delta_{\text{Search-R1}}$, $\Delta_{\text{AceSearcher}}$, $\Delta_{\text{Qwen3}}$, and $\Delta_{\text{Llama}}$. Table~\ref{tab:sensitivity_as} demonstrates that the final average performance remains robust regardless of the geometric anchor chosen.

\begin{table}[htbp]
    \centering
    \small
    \caption{Sensitivity to Reference Choice on Agentic Search. Performance of the aggregated model when rotating the reference base among the four agentic preference deltas.}
    \begin{tabular}{@{} l c @{}}
        \toprule
        \textbf{Reference Choice ($\phi_1$)} & \textbf{Agentic Search Avg. Score} \\
        \midrule
        $\Delta_{\text{Search-R1}}$ (Search-R1 preference delta)     & 58.0 \\
        $\Delta_{\text{AceSearcher}}$ (AceSearcher preference delta) & 58.0 \\
        $\Delta_{\text{Qwen3}}$ (Qwen3 preference delta)             & 57.9 \\
        $\Delta_{\text{Llama}}$ (Llama-3.2 preference delta)         & 58.0 \\
        \midrule
        \textbf{Max Fluctuation} & \textbf{$\le$ 0.1} \\
        \bottomrule
    \end{tabular}
    \label{tab:sensitivity_as}
\end{table}

Across both diverse task domains and varying numbers of adapters, the geometric consistency of GAM maintains the intrinsic relative orientations between the preference deltas. This indicates that GAM is structurally robust to the initial selection of the reference coordinate system.

\section{Degradation from Standard Supervised Fine-Tuning}
\label{app:sft}

\begin{table*}[h]
    \centering
    \setlength{\tabcolsep}{3.5pt} 
    \caption{Performance comparison between Standard SFT and preference tuning using identical source data. SFT on the chosen responses forces the strong student model to mimic a suboptimal policy, degrading its capabilities. Conversely, preference tuning effectively extracts the relative capability delta, improving overall performance.}
    \resizebox{\textwidth}{!}{%
    \begin{tabular}{@{} l ccc cccc c ccc cccc c @{}}
        \toprule
        \multirow{3}{*}{\textbf{Method}} &
        \multicolumn{8}{c}{\textbf{Qwen3-8B}} &
        \multicolumn{8}{c}{\textbf{Tülu3-8B}} \\
        \cmidrule(lr){2-9} \cmidrule(l){10-17} 
        & \multicolumn{3}{c}{\textbf{Single-hop QA}} & \multicolumn{4}{c}{\textbf{Multi-hop QA}} & \multirow{2}{*}{\textbf{Avg.}}
        & \multicolumn{3}{c}{\textbf{Single-hop QA}} & \multicolumn{4}{c}{\textbf{Multi-hop QA}} & \multirow{2}{*}{\textbf{Avg.}} \\
        \cmidrule(lr){2-4} \cmidrule(lr){5-8} \cmidrule(lr){10-12} \cmidrule(lr){13-16} 
        & \textbf{NQ} & \textbf{TQA} & \textbf{Pop.}
        & \textbf{Hot.} & \textbf{2Wiki} & \textbf{Bam.} & \textbf{MuSi.} & 
        & \textbf{NQ} & \textbf{TQA} & \textbf{Pop.}
        & \textbf{Hot.} & \textbf{2Wiki} & \textbf{Bam.} & \textbf{MuSi.} & \\
        \midrule
        Student Baseline
        & 46.2 & 61.4 & 44.6 & 53.8 & 58.1 & 61.6 & 28.9 & 50.7
        & 45.8 & 60.5 & 43.0 & 52.4 & 55.1 & 58.8 & 25.6 & 48.7 \\
        \midrule
        \multicolumn{17}{@{}l}{\textit{Training on identical Qwen3 (4B over 1.7B) source data}} \\
        \quad Standard SFT (Chosen only)
        & 44.6 & 60.8 & 43.1 & 51.9 & 54.8 & 58.8 & 25.8 & 48.5 
        & 41.7 & 56.2 & 37.4 & 48.6 & 51.2 & 51.6 & 19.7 & 43.8 \\
        \quad Single Preference Delta
        & \textbf{48.5} & \textbf{62.8} & \textbf{47.0} & \textbf{54.2} & \textbf{61.5} & \textbf{63.5} & \textbf{32.8} & \textbf{52.9} 
        & \textbf{46.8} & \textbf{62.0} & \textbf{46.8} & \textbf{53.5} & \textbf{59.2} & \textbf{62.1} & \textbf{28.5} & \textbf{51.3} \\
        \bottomrule
    \end{tabular}%
    }
    
    \label{tab:appendix_sft}
\end{table*}

In our main experiments, we claim that the performance gains achieved by our framework reflect genuine capability elicitation from preference deltas, rather than simply benefiting from exposure to the task format or identical data. To empirically validate this, we conduct a control experiment using standard Supervised Fine-Tuning (SFT) under the same parameter-efficient setup.

Specifically, we extract the ``chosen'' responses from the Qwen3 (4B over 1.7B) preference pair dataset. Instead of applying preference optimization, we use standard SFT to train the strong student models (Qwen3-8B and Tülu3-8B) to directly mimic these chosen responses via behavioral cloning. 

\paragraph{Implementation Details for SFT Baseline.} 
To ensure a rigorous and fair comparison, we maintain identical hyperparameter settings for both the SFT baseline and our preference tuning experiments. Specifically, we use the LoRA \citep{hu2021loralowrankadaptationlarge} framework with a rank of $r=64$ and $\alpha=128$, targeting all linear layers. The models are trained using the AdamW optimizer with a peak learning rate of $1.0 \times 10^{-5}$ and a cosine learning rate scheduler. We use a total batch size of 128 and train for 1 epoch, which we found to be sufficient for the model to converge on the formatting of the source data without excessive overfitting. All experiments are conducted on 4$\times$NVIDIA H20 (96GB) GPUs using the LlamaFactory \citep{zheng2024llamafactoryunifiedefficientfinetuning} framework.

The results are presented in \Cref{tab:appendix_sft}. As expected, SFT leads to a noticeable degradation in performance across almost all tasks. For Qwen3-8B, the average F1 score drops from the baseline of 50.7 to 48.5. The degradation is even more pronounced for Tülu3-8B, which drops from 48.7 to 43.8.

This phenomenon occurs because the absolute response quality of the weaker model (4B) is inherently lower than the intrinsic capabilities of the strong student models (8B). SFT encourages the strong students to mimic the suboptimal reasoning trajectories and stylistic idiosyncrasies of the weaker model, leading to capability regression. In contrast, preference optimization on the identical data focuses on the \textit{relative directional delta}. This allows the strong student model to extract the underlying capability improvements without being bottlenecked by the weak model's absolute generation quality, successfully raising the average F1 to 52.9 and 51.3 for Qwen3-8B and Tülu3-8B, respectively.

\clearpage

\section*{NeurIPS Paper Checklist}

\begin{enumerate}

\item {\bf Claims}
    \item[] Question: Do the main claims made in the abstract and introduction accurately reflect the paper's contributions and scope?
    \item[] Answer: \answerYes{} %
    \item[] Justification: We clearly state our contributions, including the \ours framework and the \ourmerging method, in the Abstract and \Cref{sec:introduction}, and support them with empirical results in \Cref{sec:experiments}. The claim about \ourmerging is stated per number of deltas, including the setting in which TIES-Merging is stronger.
    \item[] Guidelines:
    \begin{itemize}
        \item The answer \answerNA{} means that the abstract and introduction do not include the claims made in the paper.
        \item The abstract and/or introduction should clearly state the claims made, including the contributions made in the paper and important assumptions and limitations. A \answerNo{} or \answerNA{} answer to this question will not be perceived well by the reviewers. 
        \item The claims made should match theoretical and experimental results, and reflect how much the results can be expected to generalize to other settings. 
        \item It is fine to include aspirational goals as motivation as long as it is clear that these goals are not attained by the paper. 
    \end{itemize}

\item {\bf Limitations}
    \item[] Question: Does the paper discuss the limitations of the work performed by the authors?
    \item[] Answer: \answerYes{} %
    \item[] Justification: We discuss the limitations in \Cref{sec:limitation}. They cover the noise in size-based preference labels, the absence of a theoretical account of when deltas compose, the restriction of the agentic seed study to MuSiQue, the use of one source task per domain, and the untested training-based alternatives.
    \item[] Guidelines:
    \begin{itemize}
        \item The answer \answerNA{} means that the paper has no limitation while the answer \answerNo{} means that the paper has limitations, but those are not discussed in the paper. 
        \item The authors are encouraged to create a separate ``Limitations'' section in their paper.
        \item The paper should point out any strong assumptions and how robust the results are to violations of these assumptions (e.g., independence assumptions, noiseless settings, model well-specification, asymptotic approximations only holding locally). The authors should reflect on how these assumptions might be violated in practice and what the implications would be.
        \item The authors should reflect on the scope of the claims made, e.g., if the approach was only tested on a few datasets or with a few runs. In general, empirical results often depend on implicit assumptions, which should be articulated.
        \item The authors should reflect on the factors that influence the performance of the approach. For example, a facial recognition algorithm may perform poorly when image resolution is low or images are taken in low lighting. Or a speech-to-text system might not be used reliably to provide closed captions for online lectures because it fails to handle technical jargon.
        \item The authors should discuss the computational efficiency of the proposed algorithms and how they scale with dataset size.
        \item If applicable, the authors should discuss possible limitations of their approach to address problems of privacy and fairness.
        \item While the authors might fear that complete honesty about limitations might be used by reviewers as grounds for rejection, a worse outcome might be that reviewers discover limitations that aren't acknowledged in the paper. The authors should use their best judgment and recognize that individual actions in favor of transparency play an important role in developing norms that preserve the integrity of the community. Reviewers will be specifically instructed to not penalize honesty concerning limitations.
    \end{itemize}

\item {\bf Theory assumptions and proofs}
    \item[] Question: For each theoretical result, does the paper provide the full set of assumptions and a complete (and correct) proof?
    \item[] Answer: \answerYes{} %
    \item[] Justification: The only theoretical statements are two elementary properties of \ourmerging in Appendix~\ref{app:gam_math}. Each states its assumption explicitly and is followed by a complete derivation. The paper makes no other theoretical claims and does not present error bounds.
    \item[] Guidelines:
    \begin{itemize}
        \item The answer \answerNA{} means that the paper does not include theoretical results. 
        \item All the theorems, formulas, and proofs in the paper should be numbered and cross-referenced.
        \item All assumptions should be clearly stated or referenced in the statement of any theorems.
        \item The proofs can either appear in the main paper or the supplemental material, but if they appear in the supplemental material, the authors are encouraged to provide a short proof sketch to provide intuition. 
        \item Inversely, any informal proof provided in the core of the paper should be complemented by formal proofs provided in appendix or supplemental material.
        \item Theorems and Lemmas that the proof relies upon should be properly referenced. 
    \end{itemize}

    \item {\bf Experimental result reproducibility}
    \item[] Question: Does the paper fully disclose all the information needed to reproduce the main experimental results of the paper to the extent that it affects the main claims and/or conclusions of the paper (regardless of whether the code and data are provided or not)?
    \item[] Answer: \answerYes{} %
    \item[] Justification: We provide detailed descriptions of our experimental setup, prompt templates, baseline implementations, and hyperparameter choices in \Cref{sec:experiments} and Appendix \ref{app:training_details}.
    \item[] Guidelines:
    \begin{itemize}
        \item The answer \answerNA{} means that the paper does not include experiments.
        \item If the paper includes experiments, a \answerNo{} answer to this question will not be perceived well by the reviewers: Making the paper reproducible is important, regardless of whether the code and data are provided or not.
        \item If the contribution is a dataset and\slash or model, the authors should describe the steps taken to make their results reproducible or verifiable. 
        \item Depending on the contribution, reproducibility can be accomplished in various ways. For example, if the contribution is a novel architecture, describing the architecture fully might suffice, or if the contribution is a specific model and empirical evaluation, it may be necessary to either make it possible for others to replicate the model with the same dataset, or provide access to the model. In general. releasing code and data is often one good way to accomplish this, but reproducibility can also be provided via detailed instructions for how to replicate the results, access to a hosted model (e.g., in the case of a large language model), releasing of a model checkpoint, or other means that are appropriate to the research performed.
        \item While NeurIPS does not require releasing code, the conference does require all submissions to provide some reasonable avenue for reproducibility, which may depend on the nature of the contribution. For example
        \begin{enumerate}
            \item If the contribution is primarily a new algorithm, the paper should make it clear how to reproduce that algorithm.
            \item If the contribution is primarily a new model architecture, the paper should describe the architecture clearly and fully.
            \item If the contribution is a new model (e.g., a large language model), then there should either be a way to access this model for reproducing the results or a way to reproduce the model (e.g., with an open-source dataset or instructions for how to construct the dataset).
            \item We recognize that reproducibility may be tricky in some cases, in which case authors are welcome to describe the particular way they provide for reproducibility. In the case of closed-source models, it may be that access to the model is limited in some way (e.g., to registered users), but it should be possible for other researchers to have some path to reproducing or verifying the results.
        \end{enumerate}
    \end{itemize}

\item {\bf Open access to data and code}
    \item[] Question: Does the paper provide open access to the data and code, with sufficient instructions to faithfully reproduce the main experimental results, as described in supplemental material?
    \item[] Answer: \answerYes{} %
    \item[] Justification: Code is released at \url{https://github.com/AlbertQiSun/Preference-Delta-Aggregation}. Implementation and experimental settings are described in \Cref{sec:experiments} and Appendix~\ref{app:training_details}.
    \item[] Guidelines:
    \begin{itemize}
        \item The answer \answerNA{} means that paper does not include experiments requiring code.
        \item Please see the NeurIPS code and data submission guidelines (\url{https://neurips.cc/public/guides/CodeSubmissionPolicy}) for more details.
        \item While we encourage the release of code and data, we understand that this might not be possible, so \answerNo{} is an acceptable answer. Papers cannot be rejected simply for not including code, unless this is central to the contribution (e.g., for a new open-source benchmark).
        \item The instructions should contain the exact command and environment needed to run to reproduce the results. See the NeurIPS code and data submission guidelines (\url{https://neurips.cc/public/guides/CodeSubmissionPolicy}) for more details.
        \item The authors should provide instructions on data access and preparation, including how to access the raw data, preprocessed data, intermediate data, and generated data, etc.
        \item The authors should provide scripts to reproduce all experimental results for the new proposed method and baselines. If only a subset of experiments are reproducible, they should state which ones are omitted from the script and why.
        \item At submission time, to preserve anonymity, the authors should release anonymized versions (if applicable).
        \item Providing as much information as possible in supplemental material (appended to the paper) is recommended, but including URLs to data and code is permitted.
    \end{itemize}

\item {\bf Experimental setting/details}
    \item[] Question: Does the paper specify all the training and test details (e.g., data splits, hyperparameters, how they were chosen, type of optimizer) necessary to understand the results?
    \item[] Answer: \answerYes{} %
    \item[] Justification: All necessary training details, including learning rates, batch sizes, LoRA configurations, and optimizer setups, are comprehensively documented in Appendix \ref{app:training_details}, Appendix \ref{app:prompts} and \Cref{sec:experiments}.
    \item[] Guidelines:
    \begin{itemize}
        \item The answer \answerNA{} means that the paper does not include experiments.
        \item The experimental setting should be presented in the core of the paper to a level of detail that is necessary to appreciate the results and make sense of them.
        \item The full details can be provided either with the code, in appendix, or as supplemental material.
    \end{itemize}

\item {\bf Experiment statistical significance}
    \item[] Question: Does the paper report error bars suitably and correctly defined or other appropriate information about the statistical significance of the experiments?
    \item[] Answer: \answerYes{} %
    \item[] Justification: Appendix~\ref{app:seeds} reports three training seeds for the comparison between \ourmerging and TIES-Merging, including individual scores and paired differences with their means and sample standard deviations. Data and evaluation settings are fixed. The agentic training-seed study covers MuSiQue only, as stated in \Cref{sec:main_results}. Appendix~\ref{app:seeds} also reports 95\% paired bootstrap confidence intervals from 10,000 resamples for two fixed-model comparisons, with evaluation-set uncertainty distinguished from training variance. Other table entries are single runs.
    \item[] Guidelines:
    \begin{itemize}
        \item The answer \answerNA{} means that the paper does not include experiments.
        \item The authors should answer \answerYes{} if the results are accompanied by error bars, confidence intervals, or statistical significance tests, at least for the experiments that support the main claims of the paper.
        \item The factors of variability that the error bars are capturing should be clearly stated (for example, train/test split, initialization, random drawing of some parameter, or overall run with given experimental conditions).
        \item The method for calculating the error bars should be explained (closed form formula, call to a library function, bootstrap, etc.)
        \item The assumptions made should be given (e.g., Normally distributed errors).
        \item It should be clear whether the error bar is the standard deviation or the standard error of the mean.
        \item It is OK to report 1-sigma error bars, but one should state it. The authors should preferably report a 2-sigma error bar than state that they have a 96\% CI, if the hypothesis of Normality of errors is not verified.
        \item For asymmetric distributions, the authors should be careful not to show in tables or figures symmetric error bars that would yield results that are out of range (e.g., negative error rates).
        \item If error bars are reported in tables or plots, the authors should explain in the text how they were calculated and reference the corresponding figures or tables in the text.
    \end{itemize}

\item {\bf Experiments compute resources}
    \item[] Question: For each experiment, does the paper provide sufficient information on the computer resources (type of compute workers, memory, time of execution) needed to reproduce the experiments?
    \item[] Answer: \answerYes{} %
    \item[] Justification: Appendix~\ref{app:training_details} describes the training and evaluation hardware and settings. Appendix~\ref{app:merge_cost} reports merge timing on an A800 for dense and factor-based implementations, whole-model merge times, and the size of the largest intermediate tensor. The latter is distinguished from total peak GPU memory.
    \item[] Guidelines:
    \begin{itemize}
        \item The answer \answerNA{} means that the paper does not include experiments.
        \item The paper should indicate the type of compute workers CPU or GPU, internal cluster, or cloud provider, including relevant memory and storage.
        \item The paper should provide the amount of compute required for each of the individual experimental runs as well as estimate the total compute. 
        \item The paper should disclose whether the full research project required more compute than the experiments reported in the paper (e.g., preliminary or failed experiments that didn't make it into the paper). 
    \end{itemize}
    
\item {\bf Code of ethics}
    \item[] Question: Does the research conducted in the paper conform, in every respect, with the NeurIPS Code of Ethics \url{https://neurips.cc/public/EthicsGuidelines}?
    \item[] Answer: \answerYes{} %
    \item[] Justification: Our research complies strictly with the NeurIPS Code of Ethics.
    \item[] Guidelines:
    \begin{itemize}
        \item The answer \answerNA{} means that the authors have not reviewed the NeurIPS Code of Ethics.
        \item If the authors answer \answerNo, they should explain the special circumstances that require a deviation from the Code of Ethics.
        \item The authors should make sure to preserve anonymity (e.g., if there is a special consideration due to laws or regulations in their jurisdiction).
    \end{itemize}

\item {\bf Broader impacts}
    \item[] Question: Does the paper discuss both potential positive societal impacts and negative societal impacts of the work performed?
    \item[] Answer: \answerNA{} %
\item[] Justification: Our work presents a general training and model-merging method for improving LLM reasoning and search performance, without introducing a specific deployment domain or application involving direct societal impact. It does not involve sensitive personal data, safety-critical decision making, or targeted high-risk use cases. Therefore, we consider broader societal impacts to be limited and not uniquely attributable to this work beyond those generally associated with large language models.
    \item[] Guidelines:
    \begin{itemize}
        \item The answer \answerNA{} means that there is no societal impact of the work performed.
        \item If the authors answer \answerNA{} or \answerNo, they should explain why their work has no societal impact or why the paper does not address societal impact.
        \item Examples of negative societal impacts include potential malicious or unintended uses (e.g., disinformation, generating fake profiles, surveillance), fairness considerations (e.g., deployment of technologies that could make decisions that unfairly impact specific groups), privacy considerations, and security considerations.
        \item The conference expects that many papers will be foundational research and not tied to particular applications, let alone deployments. However, if there is a direct path to any negative applications, the authors should point it out. For example, it is legitimate to point out that an improvement in the quality of generative models could be used to generate Deepfakes for disinformation. On the other hand, it is not needed to point out that a generic algorithm for optimizing neural networks could enable people to train models that generate Deepfakes faster.
        \item The authors should consider possible harms that could arise when the technology is being used as intended and functioning correctly, harms that could arise when the technology is being used as intended but gives incorrect results, and harms following from (intentional or unintentional) misuse of the technology.
        \item If there are negative societal impacts, the authors could also discuss possible mitigation strategies (e.g., gated release of models, providing defenses in addition to attacks, mechanisms for monitoring misuse, mechanisms to monitor how a system learns from feedback over time, improving the efficiency and accessibility of ML).
    \end{itemize}
    
\item {\bf Safeguards}
    \item[] Question: Does the paper describe safeguards that have been put in place for responsible release of data or models that have a high risk for misuse (e.g., pre-trained language models, image generators, or scraped datasets)?
    \item[] Answer: \answerNA{} %
    \item[] Justification: Our work focuses on a parameter merging methodology and utilizes existing open-source datasets and models. It does not introduce or release new high-risk foundational models or scraped datasets.
    \item[] Guidelines:
    \begin{itemize}
        \item The answer \answerNA{} means that the paper poses no such risks.
        \item Released models that have a high risk for misuse or dual-use should be released with necessary safeguards to allow for controlled use of the model, for example by requiring that users adhere to usage guidelines or restrictions to access the model or implementing safety filters. 
        \item Datasets that have been scraped from the Internet could pose safety risks. The authors should describe how they avoided releasing unsafe images.
        \item We recognize that providing effective safeguards is challenging, and many papers do not require this, but we encourage authors to take this into account and make a best faith effort.
    \end{itemize}

\item {\bf Licenses for existing assets}
    \item[] Question: Are the creators or original owners of assets (e.g., code, data, models), used in the paper, properly credited and are the license and terms of use explicitly mentioned and properly respected?
    \item[] Answer: \answerYes{} %
    \item[] Justification: All pre-trained models (e.g., Qwen, Llama) and datasets used in this work are publicly available. We properly cite their original papers and strictly adhere to their respective usage licenses.
    \item[] Guidelines:
    \begin{itemize}
        \item The answer \answerNA{} means that the paper does not use existing assets.
        \item The authors should cite the original paper that produced the code package or dataset.
        \item The authors should state which version of the asset is used and, if possible, include a URL.
        \item The name of the license (e.g., CC-BY 4.0) should be included for each asset.
        \item For scraped data from a particular source (e.g., website), the copyright and terms of service of that source should be provided.
        \item If assets are released, the license, copyright information, and terms of use in the package should be provided. For popular datasets, \url{paperswithcode.com/datasets} has curated licenses for some datasets. Their licensing guide can help determine the license of a dataset.
        \item For existing datasets that are re-packaged, both the original license and the license of the derived asset (if it has changed) should be provided.
        \item If this information is not available online, the authors are encouraged to reach out to the asset's creators.
    \end{itemize}

\item {\bf New assets}
    \item[] Question: Are new assets introduced in the paper well documented and is the documentation provided alongside the assets?
    \item[] Answer: \answerYes{} %
    \item[] Justification: The new assets are our codebase and merging scripts, released with documentation at \url{https://github.com/AlbertQiSun/Preference-Delta-Aggregation}.
    \item[] Guidelines:
    \begin{itemize}
        \item The answer \answerNA{} means that the paper does not release new assets.
        \item Researchers should communicate the details of the dataset\slash code\slash model as part of their submissions via structured templates. This includes details about training, license, limitations, etc. 
        \item The paper should discuss whether and how consent was obtained from people whose asset is used.
        \item At submission time, remember to anonymize your assets (if applicable). You can either create an anonymized URL or include an anonymized zip file.
    \end{itemize}

\item {\bf Crowdsourcing and research with human subjects}
    \item[] Question: For crowdsourcing experiments and research with human subjects, does the paper include the full text of instructions given to participants and screenshots, if applicable, as well as details about compensation (if any)? 
    \item[] Answer: \answerNA{} %
    \item[] Justification: Our research relies purely on automated model outputs and existing datasets; it does not involve any crowdsourcing or human subjects.
    \item[] Guidelines:
    \begin{itemize}
        \item The answer \answerNA{} means that the paper does not involve crowdsourcing nor research with human subjects.
        \item Including this information in the supplemental material is fine, but if the main contribution of the paper involves human subjects, then as much detail as possible should be included in the main paper. 
        \item According to the NeurIPS Code of Ethics, workers involved in data collection, curation, or other labor should be paid at least the minimum wage in the country of the data collector. 
    \end{itemize}

\item {\bf Institutional review board (IRB) approvals or equivalent for research with human subjects}
    \item[] Question: Does the paper describe potential risks incurred by study participants, whether such risks were disclosed to the subjects, and whether Institutional Review Board (IRB) approvals (or an equivalent approval/review based on the requirements of your country or institution) were obtained?
    \item[] Answer: \answerNA{} %
    \item[] Justification: Not applicable as this research does not involve human subjects.
    \item[] Guidelines:
    \begin{itemize}
        \item The answer \answerNA{} means that the paper does not involve crowdsourcing nor research with human subjects.
        \item Depending on the country in which research is conducted, IRB approval (or equivalent) may be required for any human subjects research. If you obtained IRB approval, you should clearly state this in the paper. 
        \item We recognize that the procedures for this may vary significantly between institutions and locations, and we expect authors to adhere to the NeurIPS Code of Ethics and the guidelines for their institution. 
        \item For initial submissions, do not include any information that would break anonymity (if applicable), such as the institution conducting the review.
    \end{itemize}

\item {\bf Declaration of LLM usage}
    \item[] Question: Does the paper describe the usage of LLMs if it is an important, original, or non-standard component of the core methods in this research? Note that if the LLM is used only for writing, editing, or formatting purposes and does \emph{not} impact the core methodology, scientific rigor, or originality of the research, declaration is not required.
    \item[] Answer: \answerYes{} %
    \item[] Justification: The generation of weak-to-strong preference trajectories using existing LLMs (e.g., Qwen3-4B over 1.7B) is explicitly detailed in \Cref{sec:pda}, as it forms the basis of our Preference Delta Aggregation framework.
    \item[] Guidelines:
    \begin{itemize}
        \item The answer \answerNA{} means that the core method development in this research does not involve LLMs as any important, original, or non-standard components.
        \item Please refer to our LLM policy in the NeurIPS handbook for what should or should not be described.
    \end{itemize}

\end{enumerate}

\end{document}